%% file: root.tex
\documentclass[letterpaper, 10 pt, conference]{ieeeconf}

\IEEEoverridecommandlockouts                             
\input{./preamble_packages}

\input{./preamble_symbols}

\usepackage{capt-of}

\input{sections/shortcuts}
\usepackage{flushend}
\usepackage{tikz}

\usepackage[bookmarks=true]{hyperref}
\usepackage{adjustbox}
\usepackage[normalem]{ulem}

 \newcommand{\isExtended}[2]{#1}
 \newcommand{\suppMaterial}{Appendix\xspace}
 \newcommand{\suppMaterialLong}{\suppMaterial\xspace}

\newcommand{\removeRed}[1]{}

\renewcommand{\baselinestretch}{0.97}

\title{\LARGE \bf
 \robin: a Graph-Theoretic Approach to  Reject Outliers \\  in Robust Estimation using Invariants 
}

\author{Jingnan Shi, Heng Yang, Luca Carlone
\thanks{J.\,Shi, H.\,Yang, and L.\,Carlone are with the Laboratory for
Information \& Decision Systems (LIDS), Massachusetts Institute of Technology, Cambridge, MA 02139, USA, Email: 
{\sf \{jnshi,hankyang,lcarlone\}@mit.edu}}
\thanks{This work was partially funded by ARL DCIST CRA W911NF-17-2-0181, 
ONR RAIDER N00014-18-1-2828, Lincoln Laboratory's Resilient Perception in Degraded Environments program, and NSF CAREER award ``Certifiable Perception for Autonomous Cyber-Physical Systems''.}
}

\begin{document}

\maketitle

\begin{tikzpicture}[overlay, remember picture]
\path (current page.north east) ++(-2.2,-0.4) node[below left] {
This paper has been accepted for publication in the 2021 IEEE International Conference on Robotics and Automation.
};
\end{tikzpicture}
\begin{tikzpicture}[overlay, remember picture]
\path (current page.north east) ++(-6.5,-0.8) node[below left] {
Please cite the paper as: J. Shi, H. Yang, and L. Carlone,
};
\end{tikzpicture}
\begin{tikzpicture}[overlay, remember picture]
\path (current page.north east) ++(-3.1,-1.2) node[below left] {
``ROBIN: a Graph-Theoretic Approach to Reject Outliers in Robust Estimation using Invariants,''
};
\end{tikzpicture}
\begin{tikzpicture}[overlay, remember picture]
\path (current page.north east) ++(-4.5,-1.6) node[below left] {
    in \emph{IEEE International Conference on Robotics and Automation (ICRA)}, 2021.
};
\end{tikzpicture}

\input{sections/abstract}

\input{sections/intro}

\input{sections/related_work}

\input{sections/invariant_formulation}

\input{sections/graph_formulation}

\input{sections/robin}

\input{sections/experiments}

\input{sections/conclusion}

\isExtended{}
{
	\clearpage
}

\bibliographystyle{ieee}
\bibliography{./myRefs,./refs}

\isExtended{

\renewcommand{\thesection}{A\arabic{section}}
\renewcommand{\theequation}{A\arabic{equation}}
\renewcommand{\thetheorem}{A\arabic{theorem}}
\renewcommand{\thefigure}{A\arabic{figure}}
\renewcommand{\thetable}{A\arabic{table}}

\setcounter{equation}{0}
\setcounter{section}{0}
\setcounter{theorem}{0}
\setcounter{figure}{0}

\appendices 
\input{sections/app-proofInvariants}

\input{sections/app-invariantsInVision}

\input{sections/app-compatibilityTests}

\input{sections/app-proofGraph}
\input{sections/app-hornWithNormals}

\input{sections/app-experiments}

}{
}

\end{document}

%% file: preamble_packages.tex
\usepackage{comment}
\usepackage{siunitx}
\usepackage{relsize}
\usepackage{ifthen}
\usepackage[colorinlistoftodos]{todonotes}

\usepackage[caption=false]{subfig}

\usepackage[vlined,ruled,linesnumbered]{algorithm2e}
\usepackage{graphics} %
\usepackage{rotating}
\usepackage{color}
\usepackage{enumerate}
\usepackage[T1]{fontenc}
\usepackage{psfrag}
\usepackage{epsfig} %
\usepackage{booktabs}
\usepackage{graphicx,url}
\usepackage{multirow}
\usepackage{array}
\usepackage{latexsym}
\usepackage{amsfonts}
\usepackage{amsmath}
\usepackage{amssymb}
\usepackage{xstring}
\usepackage[noend]{algorithmic}
\usepackage{multirow}
\usepackage{xcolor}
\usepackage{prettyref}
\usepackage{flexisym}
\usepackage{bigdelim}
\usepackage{breqn} %
\usepackage{listings}
\let\labelindent\relax
\usepackage{enumitem}
\usepackage{xspace}
\usepackage{bm}
\graphicspath{{./figures/}}
\usepackage{tikz}
\usetikzlibrary{matrix,calc}

\usepackage{mdwlist}

\let\itemize\undefined
\makecompactlist{itemize}{stditemize}

%% file: preamble_symbols.tex
\newrefformat{prob}{Problem\,\ref{#1}}
\newrefformat{def}{Definition\,\ref{#1}}
\newrefformat{sec}{Section\,\ref{#1}}
\newrefformat{sub}{Section\,\ref{#1}}
\newrefformat{prop}{Proposition\,\ref{#1}}
\newrefformat{app}{Appendix\,\ref{#1}}
\newrefformat{alg}{Algorithm\,\ref{#1}}
\newrefformat{cor}{Corollary\,\ref{#1}}
\newrefformat{thm}{Theorem\,\ref{#1}}
\newrefformat{lem}{Lemma\,\ref{#1}}
\newrefformat{fig}{Fig.\,\ref{#1}}
\newrefformat{tab}{Table\,\ref{#1}}

\newtheorem{theorem}{Theorem}

\newtheorem{definition}[theorem]{Definition}
\newtheorem{proposition}[theorem]{Proposition}
\newtheorem{remark}[theorem]{Remark}

\newcommand{\bdmath}{\begin{dmath}}
\newcommand{\edmath}{\end{dmath}}
\newcommand{\beq}{\begin{equation}}
\newcommand{\eeq}{\end{equation}}
\newcommand{\bdm}{\begin{displaymath}}
\newcommand{\edm}{\end{displaymath}}
\newcommand{\bea}{\begin{eqnarray}}
\newcommand{\eea}{\end{eqnarray}}
\newcommand{\beal}{\beq \begin{array}{ll}}
\newcommand{\eeal}{\end{array} \eeq}
\newcommand{\beas}{\begin{eqnarray*}}
\newcommand{\eeas}{\end{eqnarray*}}
\newcommand{\ba}{\begin{array}}
\newcommand{\ea}{\end{array}}
\newcommand{\bit}{\begin{itemize}}
\newcommand{\eit}{\end{itemize}}
\newcommand{\ben}{\begin{enumerate}}
\newcommand{\een}{\end{enumerate}}

\newcommand{\calA}{{\cal A}}
\newcommand{\calB}{{\cal B}}
\newcommand{\calC}{{\cal C}}

\newcommand{\calE}{{\cal E}}

\newcommand{\calG}{{\cal G}}

\newcommand{\calN}{{\cal N}}

\newcommand{\calS}{{\cal S}}

\newcommand{\calV}{{\cal V}}

\newcommand{\calY}{{\cal Y}}

\newcommand{\etal}{\emph{et~al.}\xspace}

\newcommand{\eg}{\emph{e.g.,}\xspace}
\newcommand{\ie}{\emph{i.e.,}\xspace}
\newcommand{\myParagraph}[1]{{\bf #1.}\xspace}

\newcommand{\M}[1]{{\bm #1}} %
\renewcommand{\boldsymbol}[1]{{\bm #1}}

\newcommand{\hide}[1]{}
\newcommand{\wrt}{w.r.t.\xspace}

\newcommand{\hiddenText}{{\color{gray} hidden text.}}
\newcommand{\hideWithText}[1]{\hiddenText}

\newcommand{\dist}{\mathbf{dist}}

\DeclareMathOperator*{\argmin}{arg\,min}

\newcommand{\norm}[1]{\left\| #1 \right\|}

\newcommand{\tran}{^{\mathsf{T}}}

\newcommand{\diag}[1]{\mathrm{diag}\left(#1\right)}

\newcommand{\zero}{{\mathbf 0}}
\newcommand{\eye}{{\mathbf I}}

\newcommand{\matTwo}[1]{\left[\begin{array}{cc}  #1  \end{array}\right]}

\newcommand{\Real}[1]{ { {\mathbb R}^{#1} } }

\newcommand{\SOthree}{\ensuremath{\mathrm{SO}(3)}\xspace}

\newcommand{\SO}[1]{\ensuremath{\mathrm{SO}(#1)}\xspace}

\newcommand{\MM}{\M{M}}

\newcommand{\MU}{\M{U}}
\newcommand{\MR}{\M{R}}
\newcommand{\MS}{\M{S}}
\newcommand{\MI}{\M{I}}
\newcommand{\MV}{\M{V}}

\newcommand{\va}{\boldsymbol{a}} 
\newcommand{\vh}{\boldsymbol{h}} 
\newcommand{\vb}{\boldsymbol{b}}

\newcommand{\vf}{\boldsymbol{f}}

\newcommand{\vn}{\boldsymbol{n}}

\newcommand{\vp}{\boldsymbol{p}}

\newcommand{\vu}{\boldsymbol{u}}

\newcommand{\vt}{\boldsymbol{t}}
\newcommand{\vxx}{\boldsymbol{x}} 
\newcommand{\vy}{\boldsymbol{y}}

\newcommand{\vdelta}{\boldsymbol{\delta}}

\newcommand{\vnu}{\boldsymbol{\nu}}

\newcommand{\vepsilon}{\boldsymbol{\epsilon}}

\newcommand{\scenario}[1]{{\smaller \sf#1}\xspace}

\newcommand{\blue}[1]{{\color{blue}#1}}

\newcommand{\red}[1]{{\color{red}#1}}

\newcommand{\linkToPdf}[1]{\href{#1}{\blue{(pdf)}}}
\newcommand{\linkToPpt}[1]{\href{#1}{\blue{(ppt)}}}
\newcommand{\linkToCode}[1]{\href{#1}{\blue{(code)}}}
\newcommand{\linkToWeb}[1]{\href{#1}{\blue{(web)}}}
\newcommand{\linkToVideo}[1]{\href{#1}{\blue{(video)}}}
\newcommand{\linkToMedia}[1]{\href{#1}{\blue{(media)}}}
\newcommand{\award}[1]{\xspace} %

%% file: sections/shortcuts.tex
\newcommand{\HYedit}[1]{#1\xspace}
\newcommand{\myGrayMath}[1]{\text{{\color{gray}(#1)}}}

\newcommand{\barc}{\bar{c}}
\newcommand{\barcsq}{\barc^2}
\newcommand{\vm}{\boldsymbol{m}}

\newcommand{\hatmap}[1]{\left[ #1 \right]_{\times}}
\newcommand{\usphere}[1]{\mathbb{S}^{#1}}

\newcommand{\nchoosek}[2]{\left( \substack{#1 \\ #2}\right)}
\newcommand{\TLS}{\scenario{TLS}}
\newcommand{\GNC}{\scenario{GNC}}
\newcommand{\BnB}{\scenario{BnB}}

\newcommand{\ransac}{\scenario{RANSAC}}

\newcommand{\Rgt}{\MR^\circ}

\newcommand{\measured}{\vy}
\newcommand{\nrMeasurements}{N}

\newcommand{\bracket}[1]{\left[ #1 \right]}
\newcommand{\parentheses}[1]{\left( #1 \right)}
\newcommand{\cbrace}[1]{\left\{ #1 \right\}}

\renewcommand{\norm}[1]{\left\| #1 \right\|}

\newcommand{\omitted}[1]{}

\newcommand{\bmat}{\left[ \begin{array}}
\newcommand{\emat}{\end{array}\right]}

\newcommand{\teaserpp}{\scenario{TEASER++}}
\newcommand{\robinteaserpp}{\scenario{TEASER(ROBIN)++}}
\newcommand{\robin}{\scenario{ROBIN}}
\newcommand{\robinstar}{$\text{\scenario{ROBIN}}^\star$\xspace}
\newcommand{\robinstargnc}{$\text{\scenario{ROBIN}}^\star  \text{\scenario{+GNC}}$\xspace}
\newcommand{\robingnc}{\robin\scenario{+}\GNC}
\newcommand{\leegeodesic}{\scenario{Lee-Geodesic}}
\newcommand{\leechordal}{\scenario{Lee-Chordal}}
\newcommand{\FGR}{\scenario{FGR}}
\newcommand{\GoICP}{\scenario{GoICP}}
\newcommand{\ransaconemin}{\scenario{RANSAC1min}}
\newcommand{\ransactenk}{\scenario{RANSAC10K}}
\newcommand{\robinstarfgr}{$\text{\scenario{ROBIN}}^\star  \text{\scenario{+FGR}}$\xspace}
\newcommand{\robinstarhorn}{$\text{\scenario{ROBIN}}^\star  \text{\scenario{+HORN}}$\xspace}
\newcommand{\robinfgr}{$\text{\scenario{ROBIN}}  \text{\scenario{+FGR}}$\xspace}
\newcommand{\robinhorn}{$\text{\scenario{ROBIN}}  \text{\scenario{+HORN}}$\xspace}

\newcommand{\core}[1]{${#1}$-core}

\newrefformat{eq}{Eq.\,\ref{#1}}
\newcommand{\subMeas}[1]{\calS} %

\newcommand{\barvxx}{\bar{\vxx}}
\newcommand{\domainX}{\mathbb{X}}

\newcommand{\compatible}{compatible\xspace}

\newcommand{\compatibility}{compatibility\xspace}
\newcommand{\Compatibility}{Compatibility\xspace}

\newcommand{\qed}{{\hfill $\square$}}

\renewcommand{\dist}[2]{\text{dist}_{\SOthree}\left(#1,#2\right)}
\newcommand{\distsphere}[2]{\text{dist}_{\usphere{2}}\left(#1,#2\right)}
\newcommand{\abs}[1]{\left|#1\right|}
\newcommand{\newnoise}{\vdelta}
\newcommand{\crossratio}{\tau}

\newcommand{\bunny}{\scenario{Bunny}}
\newcommand{\teddyBear}{\scenario{TeddyBear}}
\newcommand{\homebrewedDB}{\scenario{HomebrewedDB}}

\newcommand{\vertices}{vertices\xspace}
\newcommand{\vertex}{vertex\xspace}

\newcommand{\robinLong}{Reject Outliers Based on INvariants} %
\newcommand{\invariant}{invariant\xspace}
\newcommand{\invariants}{invariants\xspace}
\newcommand{\Invariant}{Invariant\xspace}

\newcommand{\Exp}{\text{Exp}}

%% file: sections/abstract.tex
\begin{abstract}
Many estimation problems in robotics, computer vision, and learning 
require estimating unknown quantities in the face of outliers. 
Outliers are typically the result of incorrect data association or feature matching, and it is not uncommon to have problems 
where more than 90\% of the measurements used for estimation are outliers.
While current approaches for robust estimation (\eg~\ransac or graduated non-convexity) %
 are able to deal with moderate amounts of outliers, they fail to produce accurate estimates in the 
presence of many outliers. This paper develops an approach to prune outliers.
First, we develop a theory of \emph{invariance} that allows us to quickly check if a subset of 
measurements are mutually \compatible without explicitly solving the corresponding estimation problem.
Second, we develop a graph-theoretic framework, where measurements are modeled as \vertices and mutual compatibility is 
captured by edges in a graph. We generalize existing results showing that the 
inliers form a clique in this \emph{\compatibility graph} and 
 typically belong to the \emph{maximum clique}. %
We  also show that in practice the \emph{maximum k-core} of the \compatibility graph provides an approximation of the maximum clique, while 
being much faster to compute in large problems.
The combination of these two contributions leads to \robin, our approach to \robinLong, which 
 allows us to quickly prune outliers in generic estimation problems. %
 We demonstrate \robin in four geometric perception problems and show  
 it boosts robustness of existing solvers %
(making them robust to more than 95\% outliers), while running in milliseconds in large problems.
\end{abstract}

%% file: sections/intro.tex
\section{Introduction}
\label{sec:intro}

Robust estimation is concerned with estimating an unknown variable $\vxx$ given noisy and potentially corrupted measurements.
For instance, in 3D perception problems, the $\vxx$ may be the pose of an unknown object (\eg a car) while the measurements are features of the object (\eg wheels, headlights) detected in the sensor data (\eg images, lidar scans)~\cite{Yang20cvpr-shapeStar,Pavlakos17icra-semanticKeypoints}. In Simultaneous Localization and Mapping (SLAM),  $\vxx$ might denote the trajectory of the robot and the location of landmarks. In machine learning, 
the $\vxx$ can be a description of a model regressing the data.
In all these problems, one has to estimate the state $\vxx \in \domainX$ (where $\domainX$ is the domain of $\vxx$, for instance the set of 3D poses) given measurements $\measured_i$ and a measurement model $\vh_i$ that describes how the 
measurements are related to the unknown state $\vxx$:
\beq
\label{eq:measurements}
\measured_i = \vh_i(\vxx,\vepsilon_i), \qquad i=1,\ldots,\nrMeasurements
\eeq
where $\vepsilon_i$ denotes the measurement noise. 
In practical applications, $\measured_i$'s are the result of the preprocessing of raw sensor data: for instance, 
in 3D perception problems, they are the output of image-based or lidar-based feature detection and matching. 
This preprocessing, often referred to as {the} \emph{perception front-end}, is prone to producing many corrupted measurements (\ie \emph{outliers}) 
for which the error $\vepsilon_i$ is so large that $\measured_i$ becomes uninformative towards the estimation of $\vxx$.
For instance, traditional and deep learning techniques for feature matching in object pose estimation problems~\cite{Bustos18pami-GORE,Yang20tro-teaser} are not uncommon to produce more than 90\% outliers.

\input{sections/fig-graph-diagram}

Research in robust estimation across robotics, computer vision, and learning has made substantial progress towards 
solving estimation problems with outliers~\cite{Huber81,Bosse17fnt,Chin18eccv-robustFitting,MacTavish15crv-robustEstimation,Barron19cvpr-adaptRobustLoss}. Despite this progress, 
the literature is still divided between 
\emph{fast heuristics} and 
\emph{globally-optimally solvers}. 
Fast heuristics, such as \ransac~\cite{Fischler81} and variants~\cite{Tanaka06icra-incrementalRANSAC,Choi09cvpr-starsac,Chum05cvpr,Barath18ivpr-graphCutRANSAC,Torr00cviu,Raguram12pami-usac}, graduated non-convexity (\GNC)~\cite{Yang20ral-GNC,Black96ijcv-unification,Blake1987book-visualReconstruction},
or iterative local solvers for M-estimation~\cite{Agarwal13icra,Schonberger16cvpr-SfMRevisited} are fast but prone to fail in the presence of many outliers~\cite{Bustos18pami-GORE,Yang20tro-teaser,Raguram08-RANSACcomparative}. 
On the other hand, globally-optimally solvers, such as {Branch and Bound}~\cite{Izatt17isrr-MIPregistration,Yang16pami-goicp,Bustos18pami-GORE}, 
and combinatorial approaches for \emph{maximum consensus}~\cite{Chin17slcv-maximumConsensusAdvances}
can tolerate extreme outlier rates, but run in worst-case exponential time and are slow in practice. 
This dichotomy has been resolved only in specific problems,~\eg point cloud registration, where 
recent approaches like \teaserpp~\cite{Yang20tro-teaser} are robust to more than 99\% outliers and 
run in milliseconds. 
The goal of this paper is to provide a general framework (which indeed generalizes our approach in~\cite{Yang20tro-teaser}) 
to boost the performance of existing techniques, including \ransac and \GNC, to be robust to more than 95\% outliers while preserving their real-time operation.

\myParagraph{Contribution} We address the following key question:
\emph{Can we quickly detect and discard many outliers in the set of measurements~\eqref{eq:measurements}, while preserving all the inliers?} Intuitively, we want to bring down the number of outliers into a range where existing solvers (\ransac, \GNC) work well {(\eg~below 60\%)}.
Towards this goal, we provide two contributions.
The first contribution, presented in 
Section~\ref{sec:invariantMeasurements}, is a theory 
of \emph{invariance} that allows us to check if a subset of
$n \ll N$ measurements are mutually compatible; contrary to \ransac, our check does not require computing an estimate for $\vxx$ so it is extremely fast.
The second contribution, in Section~\ref{sec:inv-graph}, 
is\removeRed{the development of} a
graph-theoretic framework, where
\vertices represent measurements and edges represent mutual compatibility.
We generalize existing results~\cite{Mangelson18icra,Yang20tro-teaser,Enqvist09iccv}
 showing that, even for $n>2$ (related work focuses on pairwise \compatibility, where $n=2$), the inliers form a clique in this \emph{\compatibility graph} and 
 can be typically retrieved by computing the \emph{maximum clique}. 
  Despite the availability of fast algorithms,
finding a maximum clique is NP-hard and its runtime is often incompatible with real-time robotics applications.   
Therefore, we propose a faster alternative \removeRed{to finding inliers} by computing the maximum \emph{k-core} of the graph.
The combination of these contributions leads to \robin, our approach to \robinLong.\removeRed{which
  allows us to quickly prune outliers.}
We summarize \robin in Section~\ref{sec:robin}, while in Section~\ref{sec:experiments} 
 we demonstrate it in {four} perception problems 
 (single rotation averaging, point cloud registration, point-with-normal registration, and 2D-3D  camera pose estimation) and show that (i)~\robin boosts robustness of existing robust solvers, making them robust to more than 95\% outliers, and (ii)~it runs in milliseconds in large-scale simulated and real datasets. %

 Extra results, visualizations, and proofs are given in~\suppMaterialLong and the video attachment.

%% file: sections/fig-graph-diagram.tex
\begin{figure}
\vspace{-4mm}
\centering
\includegraphics[width=0.85\columnwidth]{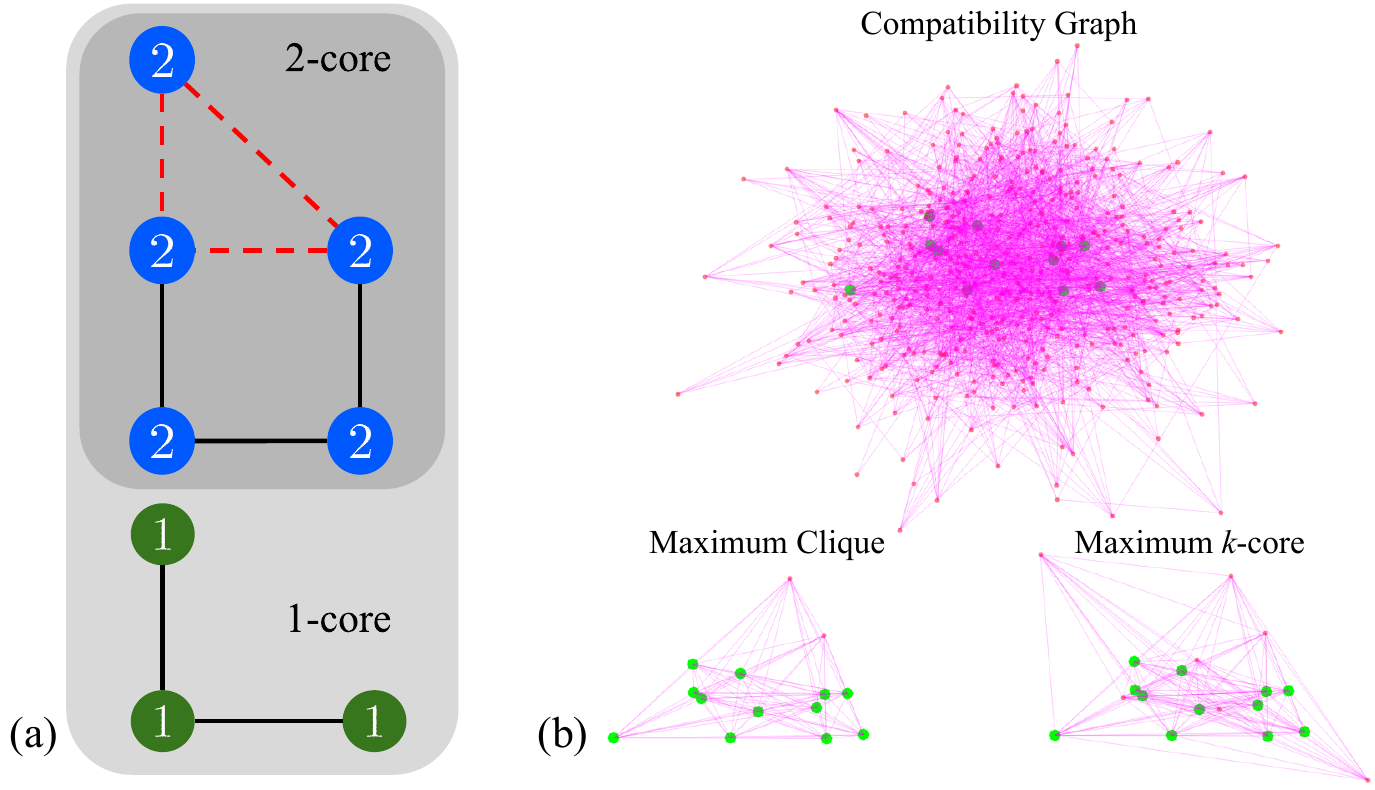}
\vspace{-3mm}
\caption{We propose \robin, a graph-theoretic approach to reject outliers in robust estimation problems. 
We use \emph{invariants} to check compatibility among subsets of measurements and describe the results of these checks using 
a \emph{\compatibility graph} $\calG$. 
\robin rejects outliers by computing a maximum clique or a maximum \core{k} of $\calG$.
(a) Maximum clique (dotted red edges) and \core{k} ($k=1,2$) in a toy graph (numbers denote the \emph{core number} of each \vertex). 
(b) \Compatibility graph %
for a single rotation averaging problem with 500 measurements and only 12 inliers (green dots).
\label{fig:graph}
\vspace{-7mm}
}
\end{figure}

%% file: sections/related_work.tex
\section{Related Work}
\label{sec:relatedWork}

{\bf Consensus Maximization} is a framework for robust estimation that aims to find {the largest} set of measurements with errors {below} a user-defined threshold.
Consensus maximization is 
NP-hard~\cite{Antonante20arxiv-outlierRobustEstimation,Chin18eccv-robustFitting},
hence the community has resorted to randomized approaches, such as \ransac~\cite{Fischler81}. 
\ransac repeatedly draws a minimal subset of measurements to compute an estimate,
and stops after finding an estimate that agrees with a large set of measurements.
While \ransac works well for problems where the minimal set is small and there are not many outliers, 
the number of iterations %
increases exponentially with the percentage of outliers~\cite{Raguram08-RANSACcomparative}, making it impractical for problems with many outliers.
On the other hand, global solvers, such as branch-and-bound (\BnB)~\cite{Li09cvpr-robustFitting} and tree search~\cite{Chin15cvpr-CMTreeAstar}, exist but scale poorly with the problem size, with \BnB being exponential in the size of $\vxx$, and tree search being exponential in the number of outliers~\cite{Cai19ICCV-CMtreeSearch}.

{\bf M-estimation} 
 performs estimation by optimizing a robust cost function that reduces the influence of the outliers. 
The resulting problems are typically optimized using iterative local solvers.
MacTavish and Barfoot~\cite{MacTavish15crv-robustEstimation} compare several robust cost functions for visual-inertial localization using  iterative re-weighted least squares solvers. The downside of local solvers is that they need a good initial guess, 
which is often unavailable in practice.
A popular approach to circumvent the need for an initial guess is \emph{Graduated Non-Convexity} (\GNC)~\cite{Blake1987book-visualReconstruction,Black96ijcv-unification}. 
Zhou~\etal~\cite{Zhou16eccv-fastGlobalRegistration} use \GNC for point cloud registration.
Yang~\etal~\cite{Yang20ral-GNC} and Antonante~\etal~\cite{Antonante20arxiv-outlierRobustEstimation} combine \GNC with global non-minimal solvers and show their general applicability to problems with up to $80\%$ outliers.

For certain low-dimensional geometric problems, fast global solvers exist.
Enqvist \etal~\cite{Enqvist12eccv-robustFitting} use a \emph{truncated least squares} (\TLS) cost to solve triangulation problems in polynomial time in the number of measurements, but exponential time in the {dimension of $\vxx$}.
Ask \etal~\cite{Ask13cvpr-optimalTruncatedL2} use a \TLS cost for image registration. %
Recently, \emph{certifiably-robust} globally-optimal solvers based on \emph{convex relaxations} have been used 
for \emph{M-estimation}~\cite{Lajoie19ral-DCGM,Yang20tro-teaser,Yang19iccv-QUASAR}.
Unfortunately, due to the poor scalability of current SDP solvers, such methods are mostly viable 
to \emph{check} optimality~\cite{Yang20tro-teaser,Yang20neurips-certifiablePerception}.

{\bf Graph algorithms} have been extensively used in robot perception and vision, 
from SLAM~\cite{Cadena16tro-SLAMsurvey,Dellaert17fnt-factorGraph,Rosinol20rss-dynamicSceneGraphs} to human-pose estimation~\cite{Bray06eccv}.
Neira and Tard\'os propose compatibility tests and \BnB for data association in SLAM~\cite{Neira01tra}. 
Bailey~\etal~\cite{Bailey00icra-dataAssociation} propose a Maximum Common Subgraph algorithm for feature matching in lidar scans.
Graph algorithms, such as graph-cut, have been used in \ransac to improve performance against noise and outliers~\cite{Barath18-gcRANSAC}. In~\cite{Perera12-maxCliqueSegmentation}, graph cliques are used to segment objects under rigid body motion.
Leordeanu and Hebert~\cite{Leordeanu05-spectral} {establish image matches by finding strongly-connected clusters in the correspondence graph with an approximate spectral method.}
Enqvist~\etal~\cite{Enqvist09iccv} develop an outlier rejection algorithm for 3D-3D and 2D-3D registration based on approximate vertex cover. %
Recently, progress in graph algorithms (see~\cite{Rossi15parallel, Parra19arXiv-practicalMaxClique}) have allowed for the emergence of fast graph-based algorithms, such as \teaserpp~\cite{Yang20tro-teaser}, that are robust to high-noise and outliers.

%% file: sections/invariant_formulation.tex
\section{Measurement Invariants}
\label{sec:invariantMeasurements}

This section develops a theory of invariance that allows checking if a subset of measurements 
contains outliers
without explicitly computing an estimate for $\vxx$.
We will then use these checks to select large sets of mutually \compatible 
measurements and reject outliers in Section~\ref{sec:inv-graph}.

\subsection{From Measurements to Invariants}

Let us consider the measurements in eq.~\eqref{eq:measurements} and 
denote the indices of the measurements as $\calY \doteq \{1,\ldots,N\}$.
For a given integer $n \leq N$, let $\subMeas{n} \subset \calY$ be a subset of $n$ indices in $\calY$, %
and denote with $\subMeas{n}_j$ the $j$-th element of this subset (with $j=1,\ldots,n)$.
Then, we use the following notation:
\beq
\measured_{\subMeas{n}} = \matTwo{ \measured_{\subMeas{n}_1} \\ \measured_{\subMeas{n}_2} \\ \vdots \\\measured_{\subMeas{n}_n}   },
\;\;
\vh_{\subMeas{n}} = \matTwo{ \vh_{\subMeas{n}_1} \\ \vh_{\subMeas{n}_2} \\ \vdots \\\vh_{\subMeas{n}_n}   },
\;\;
\vepsilon_{\subMeas{n}} = \matTwo{ \vepsilon_{\subMeas{n}_1} \\ \vepsilon_{\subMeas{n}_2} \\ \vdots \\ \vepsilon_{\subMeas{n}_n}   }
\eeq
which is simply stacking together measurements $\vy_i$, functions $\vh_i$, and noise $\vepsilon_i$ for the subset of measurements $i \in \subMeas{n}$.

We are now ready to define an \emph{$n$-\invariant}.

\begin{definition}[n-\Invariant]\label{def:n-measurement-invariant}
Given measurements~\eqref{eq:measurements}, a function $\vf$ is called an 
\emph{$n$-\invariant} if \mbox{for all $\subMeas{n} \!\subset\!\calY$ of size $n$:} 
$\vf( \measured_{\subMeas{n}}  ) \!=\! \vf( \vh_{\subMeas{n}}(\barvxx, \vdelta_{\subMeas{n}}) )$, 
 $\forall \barvxx \!\in\!\domainX$, {where $\vdelta$ is a noise term whose statistics do not depend on $\barvxx$, 
 and $\vdelta\!=\!\zero$ whenever $\vepsilon\!=\!\zero$ in~\eqref{eq:measurements} (\ie when measurements are noiseless).}
\end{definition}

Intuitively, the invariant function $\vf$ takes a subset of measurements and computes 
a quantity that no longer depends on $\vxx$ (hence the equation is satisfied for all $\vxx \in \domainX$).
When $n=2$ in Definition~\ref{def:n-measurement-invariant}, we obtain a \emph{pairwise \invariant}.

\subsection{Example of \invariants}
\label{sec:examples}

Here we provide examples of \invariants for {four} estimation problems in robotics and vision: single rotation averaging, point cloud registration, point-with-normal registration, and 2D-3D camera pose estimation.

{\bf Single rotation averaging}~\cite{Hartley13ijcv} seeks to   
find an~unknown rotation $\MR\!\in\!\SOthree$ given noisy measurements $\MR_i\!\in\!\SOthree$: 
\begin{equation} \label{eq:rot-avg-gen-model}
  \MR_{i} = %
  \MR \cdot \Exp(\vepsilon_{i})
\end{equation}
where $\vepsilon_{i} \in \Real{3}$ denotes measurement noise and $\Exp$ maps a 3D vector to a rotation matrix~\cite[\S 7.1.3]{Barfoot17book}. 
Eq.~\eqref{eq:rot-avg-gen-model} has the same form of~\eqref{eq:measurements}, with $\vy_i \!\doteq\! \MR_i$, $\vxx \!\doteq\! \MR$, and 
$\vh_i(\MR,\vepsilon_{i}) \doteq \MR \cdot \Exp(\vepsilon_{i})$. 
The problem admits a simple pairwise \invariant.

\begin{proposition}[Pairwise Invariant for Rotation Averaging]\label{prop:rotAve-invariants}
The function $\vf(\MR_i, \MR_j) \doteq \MR_i\tran \MR_j$, where $\MR_i,\MR_j$ are a pair of measurements, is an invariant for~\eqref{eq:rot-avg-gen-model} and satisfies:
\beq
\label{eq:rotAve-invariants}
\vf(\MR_i, \MR_j) \doteq \MR_i\tran \MR_j = \Exp(\vdelta_{i})\tran \Exp(\vdelta_{j})
\eeq 
with $\vdelta_i \!=\! \vepsilon_i$ and $\vdelta_j \!=\! \vepsilon_j$, for any rotation $\MR$ in~\eqref{eq:rot-avg-gen-model}. 
\qed
\end{proposition}

The proof is trivial and is given in \suppMaterial for completeness.

{\bf Point cloud registration}~\cite{Arun87pami,Horn87josa} consists in finding a rigid body transformation $(\MR,\vt)$ (where $\MR \in \SOthree$ and $\vt \in \Real{3}$) that 
aligns two sets of 3D points $\va_i \in \Real{3}$ and $\vb_i \in \Real{3}$, with $i=1,\ldots,N$.
The corresponding measurement model can again be phrased in the form of~\eqref{eq:measurements}: 
\begin{equation}
  \label{eq:p-reg-gen-model}
  \vb_{i} = \MR \va_i + \vt + \vepsilon_i
\end{equation}
where $\vepsilon_i \in \Real{3}$ is the measurement noise.

\begin{proposition}[Pairwise Invariant for Point Registration]\label{prop:p-reg-invariants}
Assuming the noise $\vepsilon_i$ is isotropic,\footnote{Isotropic noise has a rotation-invariant distribution, for instance, a zero-mean Gaussian with covariance matrix equal to $\sigma^2 \eye_3$.} the function $\vf(\vb_i, \vb_j) \!\doteq\! \| \vb_j - \vb_i \|$, where $\vb_i, \vb_j$ are a pair of measurements, is an invariant for~\eqref{eq:p-reg-gen-model} and satisfies:
\beq
\label{eq:p-reg-invariants}
\vf(\vb_i, \vb_j) \doteq \| \vb_j\!-\!\vb_i \| =  \| \va_j  - \va_i + \vdelta_j - \vdelta_i\|
\eeq 
\mbox{for any transformation\,$(\MR,\vt)$ between the points in~\eqref{eq:p-reg-gen-model},\,where} 
$\vdelta_i$ (resp.\,$\vdelta_j$) has the same distribution of $\vepsilon_i$ (resp.\,$\vepsilon_j$).\qed
\end{proposition}

The proposition formalizes the intuition that the distance between points in a point cloud is invariant 
to rigid transformations and has been used to detect outliers in~\cite{Enqvist09iccv,Yang19rss-teaser,Bustos2015iccv-gore3D}.
Note that in this problem (and in the ones below), each measurement encodes an association 
between two points in the two point clouds and for this reason the measurements are often referred to as 
\emph{correspondences} or \emph{putative matches}. %

{\bf Point-with-normal  registration}
 is similar to point-cloud registration in that it looks for a rigid body transformation $(\MR,\vt)$ 
that aligns 3D points, except the 3D points $\va_i$ (resp. $\vb_i$) have normals 
$\vn^a_i$ (resp. $\vn^b_i$), with $\vn^a_i, \vn^b_i \in \Real{3}$ and $\|\vn^a_i\|=\|\vn^b_i\|=1$.
This results in the measurement model: %
\begin{align}
  \label{eq:pwithn-reg-gen-model}
\vb_{i} = \MR \va_i + \vt + \vepsilon_i, \qquad
\vn^b_i &= \MR \cdot \Exp(\vnu_i) \cdot \vn^a_i
\end{align}
where $\vnu_i \in \Real{3}$ is noise perturbing the measured normals.
\begin{proposition}[Pairwise Invariant for P\&N Registration]\label{prop:pwithn-reg-invariants}
Assuming the noise $\vepsilon_i$ is isotropic, the function 
$\vf(\vy_i, \vy_j)
\doteq [\| \vb_j - \vb_i \|, (\vn^b_i)\tran \vn^b_j ]\tran$, where 
 $\vy_i \doteq [(\vb_{i})\tran, (\vn^b_i)\tran]\tran$ 
 and $\vy_j \doteq [(\vb_{j})\tran, (\vn^b_j)\tran]\tran$
 are a pair of measurements in~\eqref{eq:pwithn-reg-gen-model},
  is an invariant for~\eqref{eq:pwithn-reg-gen-model} and satisfies: 
  \beq
\label{eq:pwithn-reg-invariants}
\hspace{-3mm}\vf(\vy_i, \vy_j)
\!\doteq\! \matTwo{ \!\! \| \vb_j - \vb_i \| \!\!\!\! \\ \!\! (\vn^b_i)\tran \vn^b_j \!\!\!\! }
\!\!=\!\! \matTwo{  \!\!\| \va_j  - \va_i + \vdelta^p_j - \vdelta^p_i\| \!\!\!\!
 \\ 
\!\!(\vn^a_i)\tran  \Exp(\vdelta^n_i)\tran \Exp(\vdelta^n_j) \vn^a_j \!\! \!\!}
\eeq 
for any rigid transformation $(\MR,\vt)$ in~\eqref{eq:pwithn-reg-gen-model}, 
where $\vdelta^p_i$ has the same distribution of $\vepsilon_i$, and 
$\vdelta^n_i = \vnu_i$, $i=1,\ldots,N$. \qed
\end{proposition}

Proposition~\ref{prop:pwithn-reg-invariants} formalizes the intuition that the angle between normals, measured by the dot product $(\vn^b_i)\tran \vn^b_j$, is invariant to rigid transformations.

{\bf 2D-3D camera pose estimation} consists in finding the 3D pose of a camera
 $(\MR,\vt)$ with respect to a 3D object of known shape from 
 pixel projections of points of the object. 
 Calling the pixel projections $\vy_i \in \Real{2}$ and the corresponding 3D points $\vp_i \in \Real{3}$,
 the measurement model becomes: 
\begin{equation}
  \label{eq:proj-gen-model}
  \vy_{i} = \pi(\MR,\vt,\vp_i)+\vepsilon_i
\end{equation}
where $\pi(\cdot)$ is the standard perspective projection~\cite[Section 6.2]{Hartley04}, 
and $\vepsilon_i$ is the measurement noise.
While it is known that there is no invariant for 3D points in generic configuration under perspective projection~\cite{Mundy92book}, there exist invariants for specific configurations of points (\eg collinear or coplanar points).
For instance, we can design a 3-\invariant capturing that projections of collinear 3D points are also collinear. 
Here we review another invariant for the case %
of collinear features,  which is known as \emph{cross ratio} in computer vision~\cite[p. 45]{Hartley04}  

\begin{proposition}[4-\invariant for Camera Pose]\label{prop:cross-invariants}
Given four collinear 3D points $\vp_i \in \Real{3}$ ($i=1,\ldots,4$) and the corresponding pixel projections ${\vy}_{i} \in \Real{2}$, the \emph{cross ratio}~\cite[p. 45]{Hartley04}
is a 4-\invariant and satisfies:
\newcommand{\myMinus}{\!-\!}  
\newcommand{\myPlus}{\!+\!}
\begin{align}
\label{eq:cross-invariants}
\vf({\vy}_{1},{\vy}_{2},{\vy}_{3},{\vy}_{4}) \doteq 
\frac{\| {\vy}_{1} \myMinus \newnoise_1 \myMinus {\vy}_{2} \myPlus \newnoise_2 \| \| {\vy}_{3} \myMinus \newnoise_3 \myMinus {\vy}_{4} \myPlus \newnoise_4 \|}{ 
\| {\vy}_{1} \myMinus \newnoise_1 \myMinus {\vy}_{3}  \myPlus \newnoise_3 \| \| {\vy}_{2} \myMinus \newnoise_2 \myMinus {\vy}_{4} \myPlus \newnoise_4\| }  \nonumber \\
=  
\frac{ \| (\vp_1)^\vee - (\vp_2)^\vee \| \| (\vp_3)^\vee - (\vp_4)^\vee \|  }{ 
\| (\vp_1)^\vee - (\vp_3)^\vee \| \| (\vp_2)^\vee - (\vp_4)^\vee \|  } \nonumber \\
\end{align}
where $\vdelta_i = \vepsilon_i$ ($i=1,\ldots,4$), and for a 3D point $\vp = [p_x \; p_y \; p_z]\tran$,  we denoted $(\vp)^\vee \doteq [\frac{p_x}{p_z} \; \frac{p_y}{p_z}]\tran$. \qed
\end{proposition}

\begin{remark}[Other Invariants]
Invariants have been studied (typically in a noiseless setup) in the context of pattern recognition in computer vision.
Early related work provides an abundance of resources to design invariants~\cite{Gros92-projectiveInvariantsTheory}, 
that we discuss more broadly in~\suppMaterialLong.
Similar ideas have been also used in specific robotics applications, \eg SLAM~\cite{Mangelson18icra,Carlone14iros-robustPGO2D}, 
without a systematic treatment. 
\end{remark}

\subsection{From Invariants to \Compatibility for Outlier Rejection}
\label{sec:compatibilityTests}

While the previous section developed invariants without distinguishing inliers from outliers, this section shows that 
the resulting invariants can be directly used to check if a subset of measurements contains an outlier.
Towards this goal, we formalize the notion of inlier and outlier. %

\begin{definition}[Inliers and Outliers]\label{def:inlier-outlier}
Given measurements \eqref{eq:measurements} and a threshold $\beta>0$, 
a measurement $i$ is called an \emph{inlier} if the corresponding noise satisfies $\|\vepsilon_i\| \leq \!\beta$ 
and is called an \emph{outlier} otherwise.
\end{definition}

Now we note that in the previous section the notion of invariants allowed us to obtain relations that 
depend on the measurements and the noise, but are independent on $\vxx$, see eqs.~\eqref{eq:rotAve-invariants},~\eqref{eq:p-reg-invariants},~\eqref{eq:pwithn-reg-invariants},~\eqref{eq:cross-invariants}.
Therefore, we can directly use these relations to check if a subset of $n$ measurements contains outliers: 
if eqs.~\eqref{eq:rotAve-invariants},~\eqref{eq:p-reg-invariants},~\eqref{eq:pwithn-reg-invariants},~\eqref{eq:cross-invariants}
are not satisfied by any choice of noise smaller than $\beta$, then the corresponding subset of measurements \emph{must} contain an outlier.
We call the corresponding check a \emph{\compatibility test}. 
In the following, we provide an example of \compatibility test for point cloud registration.
The reader can find the compatibility tests for the other applications in~\suppMaterialLong.

\myParagraph{\Compatibility test for point cloud registration}
 Eq.~\eqref{eq:p-reg-invariants} states that any pair of measurements satisfies:
 \beq
 \label{eq:test-p-ref1}
 \| \vb_j - \vb_i \| =  \| \va_j  - \va_i + \newnoise_j - \newnoise_i\| 
 \eeq
 where $\newnoise_i,\newnoise_j$ have the same distribution of the measurement noise $\vepsilon_i,\vepsilon_j$.
 Therefore, one can check if one of the measurements $\vb_i, \vb_j$ is an outlier by checking if~\eqref{eq:test-p-ref1} 
 can be satisfied by any choice of $\|\vdelta_i\|,\|\vdelta_j\|\leq\!\beta$.
If we apply the triangle inequality to the right-hand-side of~\eqref{eq:test-p-ref1} we obtain:
  \beq
 \label{eq:test-p-ref2}
 -\| \newnoise_j - \newnoise_i\|  \leq  \| \vb_j - \vb_i \| - \| \va_j  - \va_i  \| \leq  \| \newnoise_j - \newnoise_i\| 
 \eeq
Now for two inliers we have $\|\newnoise_i\|\leq\!\beta$ and $\|\newnoise_j\|\leq\!\beta$, therefore 
$\| \newnoise_j - \newnoise_i \| \leq 2\beta$. Substituting in~\eqref{eq:test-p-ref2} we obtain: 
 \beq
 \label{eq:test-p-ref3}
 -2\beta  \leq  \| \vb_j - \vb_i \| - \| \va_j  - \va_i  \| \leq  2\beta
 \eeq
Therefore, if $\vb_i, \vb_j$ are inliers they must satisfy~\eqref{eq:test-p-ref3}, {which is easy to check, and intuitively
states that inliers associate pairs of points having similar distance in the two point clouds.}

Generalizing this example, we obtain the following definition of \emph{\compatibility test}.

\begin{definition}[\Compatibility Test]
Given a subset of $n$ measurements and the corresponding $n$-\invariant, a \compatibility test 
 is a binary condition involving the invariant, such that if the condition fails the set of measurements 
 \emph{must} contain at least {one} outlier. 
\end{definition}

Note that we require the test to be \emph{sound} (\ie it does not detect outliers when testing a set of inliers), but 
may not be \emph{complete} (\ie the test might pass even in the presence of {outliers}). 
This property is important since our goal is to prune as many outliers as we can, while preserving the inliers.
Also note that the test detects if the set contains {outliers}, but does not provide information on \emph{which} {measurements are 
outliers}. We are going to fill in this gap below.

%% file: sections/graph_formulation.tex
\section{\Compatibility Graph For Outlier Rejection}
\label{sec:inv-graph}

This section provides a graph-theoretic framework to use the \compatibility tests introduced in Section~\ref{sec:invariantMeasurements} 
and point out which measurements are outliers. 
We start by defining the notion of \emph{\compatibility graph} and then show that inliers form a clique in this graph. 
We then propose the use of maximum clique and maximum \core{k} to find inliers.

\subsection{From \Compatibility Tests to \Compatibility Graph}

For a problem with an $n$-\invariant, 
we describe the results of the compatibility tests for all subsets of $n$ measurements using a 
\emph{\compatibility graph}. Formally, we define the \compatibility graph
$\calG(\calV,\calE)$ as an undirected graph, where 
each \vertex in the \vertex set $\calV$ is associated to a measurement in~\eqref{eq:measurements} and 
an edge $(i,j)$ in the edge set $\calE$ is present in the graph if measurements $i$ and $j$ were involved in \emph{at 
least one} of the subsets of measurements that passed the compatibility tests. 
Building the \compatibility graph requires looping over all subsets of $n$ measurements and, whenever the subset 
passes the \compatibility test, adding edges among the corresponding $n$ nodes in the graph. 
Note that these checks are very fast and easy to parallelize since they only involve checking boolean conditions (\eg~\eqref{eq:test-p-ref3})  without computing an estimate (as opposed to \ransac).

\subsection{Inlier Structures in \Compatibility Graphs}

This section shows that we can prune many outliers 
in~\eqref{eq:measurements} by computing the maximum clique of 
the \compatibility graph. Let us start with some definitions.

\begin{definition}[Maximum Clique]
A clique of a graph $\calG$ is a set of vertices such that any pair of vertices is connected by an edge in $\calG$.
The \emph{maximum} clique of $\calG$ is the clique with the largest number of vertices.
The number of vertices $\omega(\calG)$ in the maximum clique 
is called the \emph{clique number} of $\calG$.
\end{definition}

Given a \compatibility graph $\calG$, the following result relates the set of inliers in~\eqref{eq:measurements} 
with cliques in $\calG$. 

\begin{theorem}[Inliers and Cliques]\label{thm:inliers-form-clique}
Assume~we are given measurements~\eqref{eq:measurements} and the corresponding $n$-\invariants;
call $\calG$ the corresponding \compatibility graph.
  Then, assuming there are at least $n$ inliers, the  inliers form a clique in $\calG$.\qed
\end{theorem}

Theorem~\ref{thm:inliers-form-clique} implies that we can look for inliers by computing cliques in the \compatibility graph.
Since we expect to have more \compatible inliers than outliers, we propose to compute the maximum clique,\omitted{\footnote{An alternative
approach is to explore all \emph{maximal} cliques, see~\suppMaterial.}} 
which is shown to work extremely well in practice in Section~\ref{sec:experiments}.
\removeRed{While modern maximum-clique algorithms are becoming faster and more scalable,
finding a maximum clique is NP-hard and its runtime %
often clashes  
with real-time robotics applications.}
 On the other hand, when the maximum clique is too expensive to compute, we propose to use the \emph{maximum \core{k}} (defined below) instead, which empirically approximates the maximum clique,
 while being much faster (linear-time~\cite{Dasari14-KCorePARK}) to compute. 

\begin{definition}[\core{k}, {Fig.~\ref{fig:graph}}]
A subgraph $\calC_k$ of $\calG$ is called a \emph{\core{k}} if it is the largest subgraph such that all its vertices 
have \vertex degree (\ie number of incident edges) at least $k$ within $\calC_k$. 
The \emph{core number} $k(i)$ of a \vertex $i$ is defined as the largest $k$ such that $i$ is in $\calC_{k}$.
The \emph{degeneracy} $k^{\star}(\calG)$ is the largest core number of all vertices in $\calG$.
\end{definition}

It is well-known that the maximum clique is included in the $(\omega(\calG)-1)$-core~\cite{Rossi15parallel}.
Although few conditions are known to guarantee that the maximum clique is inside the \emph{maximum} $k$-core, previous work~\cite{Walteros20OR-mcandmk} has observed that this is often the case in practical problems.
Indeed, Section~\ref{sec:experiments} shows that --in compatibility graphs arising in geometric perception-- the maximum \core{k} provides a great approximation to the maximum clique, while being faster to compute.

  While maximum clique has been used for outlier rejection in specific applications~\cite{Yang20tro-teaser,Mangelson18icra}, 
  our framework generalizes to $n$($\geq\!\!2$)-\invariants (rather than pairwise) and also provides a linear-time algorithm (based on \core{k}) for outlier rejection.
   The interested reader can find a broader discussion in~\suppMaterial.

%% file: sections/robin.tex
\section{\robin: Summary of Our Algorithm}
\label{sec:robin}

This section summarizes how we combine measurement invariants and graph theory to develop our outlier pruning approach, named 
 \robin (\emph{\robinLong}). 
 \robin's pseudocode is given in Algorithm~\ref{alg:robin}.
 The algorithm takes a set of measurements $\calY$ and outputs a subset $\calY^\star \subset \calY$ from which many outliers have been pruned. 
 If the problem admits an $n$-measurement invariant, \robin first performs \compatibility tests on all subsets of $n$ measurements  
 and builds the corresponding \compatibility graph (lines~\ref{line:startGraph}-\ref{line:endGraph}). 
 Then it uses maximum clique or \core{k} solvers to compute and return the subset of measurements containing the inliers 
 (lines~\ref{line:graphTheory}-\ref{line:return}). 
We have implemented \robin~in C++, using the parallel core decomposition algorithm from~\cite{Kabir17IPDPSW-KCorePKC} and parallel maximum clique solver from~\cite{Rossi15parallel}.
We remark that \robin is not guaranteed to reject all outliers.
 Indeed,  as mentioned in the introduction, \robin is designed to be a preprocessing 
step to enhance the robustness of existing robust estimators.

\setlength{\textfloatsep}{0pt}%
\begin{algorithm}[t]
{\footnotesize
\SetAlgoLined
\textbf{Input:} \ set of measurements $\calY$ and model~\eqref{eq:measurements}; $n$-\invariant function $\vf$ (for some $n$); boolean fastMode (use maximum \core{k} if true)\; 
\textbf{Output:} \  subset $\calY^\star \subset \calY$ containing inliers\;
\% Initialize \compatibility graph \label{line:startGraph} \\
$\calV  = \calY$; \!\!\quad  \% each vertex is a measurement \label{line:vertices}\\ 
$\calE  = \emptyset$; \quad \% start with empty edge set \label{line:emptyEdges}\\ 
\% Perform \compatibility tests \\
\For{{\bf all subsets } $\subMeas{n} \subset \calY$ {\bf of size $n$}}{
	\If{ {\rm testCompatibility}($\subMeas{n},\vf$) = {\rm pass} }{
		add edges between any $i,j \in \subMeas{n}$ to $\calE$;   \label{line:addEdges} \\
	}
} \label{line:endGraph}
\% Find \compatible measurements \label{line:graphTheory} \\
\eIf{ {\rm fastMode} = {\rm true}}{
	$\calY^\star = {\rm max\_kcore}(\calV,\calE)$;  \\
}{
	$\calY^\star = {\rm max\_clique}(\calV,\calE)$;  \\
}
 \textbf{return:} $\calY^\star$. \label{line:return}
 \caption{\robin \label{alg:robin}}
}
\end{algorithm}

%% file: sections/experiments.tex
\section{Experiments}
\label{sec:experiments}
We test the performance of \robin on four geometric perception problems: single rotation averaging (Section~\ref{sec:exp:singlerotationavg}), point cloud registration and point-with-normal registration 
(Section~\ref{sec:exp:pointcloudreg} and Section~\ref{sec:exp:regNoCorrespondences}),
and 2D-3D camera pose estimation (Section~\ref{sec:exp:cross-ratio}).
We show that (i) \robin boosts the robustness of \GNC to more than $95\%$ outliers and dominates the state of the art, and
 (ii) \robin runs in tens of milliseconds on challenging large-scale datasets.
In addition, we present an experiment in which the maximum \core{k} becomes overly conservative and fails to reject outliers (Section~\ref{sec:exp:regNoCorrespondences}).

We denote as \robinstar the version of \robin~that computes the maximum clique (fastMode=false in Algorithm~\ref{alg:robin}), 
while we denote as \robin the version that computes the maximum \core{k}. 
Moreover, we denote as \robinstargnc (resp. \robingnc) the estimation approach where we first 
remove outliers using \robinstar (resp. \robin), and then use \GNC~\cite{Yang20ral-GNC} to compute an estimate.
All experiments are run on a Linux computer with an Intel i9-9920X CPU at 3.5 GHz.

\input{sections/fig-all-experiments-compressed}

\subsection{Single Rotation Averaging}
\label{sec:exp:singlerotationavg}

\myParagraph{Setup} In each Monte Carlo run, we first generate a random ground-truth rotation $\Rgt \in \SOthree$. Then we create $N=1000$ measurements $\cbrace{\MR_i}_{i=1}^N$ with increasing outlier rates from $0\%$ to $99\%$. An inlier measurement $\MR_i$ is generated as $\MR_i = \Rgt \Exp(\theta_i \vu_i)$,
where $\vu_i$ is a random 3D unit vector and $\theta_i \sim \calN(0,\sigma^2)$ is a random rotation angle ($\sigma = 5^{\circ}$),
while an outlier measurement is a random 3D rotation. We compare five algorithms in 100 Monte Carlo runs: (i) \robinstargnc;
(ii) \robingnc;
(iii) \GNC (without \robin); 
(iv-v) \leechordal and \leegeodesic~\cite{Lee20arXiv-robustSRA}, 
two recent state-of-the-art
methods (we use the authors' implementations).

\myParagraph{Results} Fig.~\ref{fig:SRA-errors} boxplots the rotation errors of the five algorithms under selected outlier rates (full results can be found in \suppMaterial). Both \robinstargnc and \robingnc are robust against $98\%$ outliers, while \GNC, \leechordal, and \leegeodesic start failing at $90\%$ outliers. As expected, \robinstargnc slightly outperforms \robingnc at $99\%$ outliers, as shown by the fewer failures (``{\bf \red{+}}'') of \robinstargnc.

\subsection{Registration With Correspondences}\label{sec:exp:pointcloudreg}

\myParagraph{Simulated Benchmarks}
We use the \bunny~model from the Stanford 3D scanning repository.
We downsample the model to 1000 points and resize within a $[0,1]^{3}$ cube to obtain the source point cloud.
We then apply a random transformation $(\MR, \vt)$ with $\MR \in \SO3$ and $\|t\| \leq 1$ according to~\eqref{eq:p-reg-gen-model}.
We add bounded noise $\vepsilon_i \sim \calN(\bm{0},\sigma^2 \MI)$ 
and ensure $\|\vepsilon_i\| \leq\!\!\beta$.
We set $\sigma = 0.01$ and $\beta = 5.54 \sigma$ such that $\mathbb{P}\left(\|\vepsilon_i \|^2  > \beta^2 \right) \leq 10^{-6}$.
To generate outliers, we replace some $\vb_{i}$'s with vectors uniformly sampled in a sphere with radius $5$.
We benchmark 9 methods: 
 (i)~\FGR: Fast Global Registration~\cite{Zhou16eccv-fastGlobalRegistration}, which is a \GNC approach tailored to registration;
 (ii-iii)~\robinstarfgr, \robinfgr: \FGR with our outlier pruning;
 (iv-v)~\robinstarhorn, \robinhorn:  Horn's method~\cite{Horn87josa} with our outlier pruning;
 (vi)~\teaserpp~\cite{Yang20tro-teaser}, which is equivalent to using \robinstar~with a decoupled \GNC~solver;
 (vii)~\robinteaserpp: same as \teaserpp, except \robin (with \core{k}) is used to reject outliers; 
 (viii-ix)~\ransaconemin~and \ransactenk: two \ransac~variants (with 99\% confidence) that terminate after a maximum of one minute or 10,000 iterations, respectively.

\myParagraph{Results} Fig.~\ref{fig:3dreg-with-features} shows translation errors for increasing outlier rates 
(rotation errors are qualitatively similar and reported in~\suppMaterial).
None of the \robin-based methods fail at outlier rates below 95\%, so we omit those results for brevity.
Methods using \robinstar~and \robin~dominate \FGR and \ransac.
 \robinhorn, \robinfgr, \robinteaserpp~fail at 99\%.
Notably, \robinstar~is so robust that even Horn's method ({which is }not robust to outliers) succeeds at {98\% outliers}.
All six \robin-based methods have runtime of tens of milliseconds (average: 4ms), 
with \robinteaserpp~being the fastest, {due to its optimized C++ implementation of GNC.} %

\myParagraph{3DMatch Dataset}
We evaluate \robin on the 8 test scenes of the 3DMatch dataset~\cite{Zeng17cvpr-3dmatch}, which contains real-world indoor {RGB-D} scans.
Using FPFH descriptors, we generate feature correspondences and feed them to \teaserpp and \robinteaserpp.
We also compute the normals using Open3D~\cite{Zhou18arxiv-open3D} and
feed the point-with-normal correspondences to \robinstargnc~and \robingnc~(the description of the
corresponding \GNC solver can be found in~\suppMaterial).

\myParagraph{Results} Table~\ref{tab:scanMatching} shows the percentage of successful registrations for each scene.
A registration is considered successful if the 
 estimated transformation has (i) rotation error $\leq 15^{\circ}$, and (ii) translation error $\leq 30$cm (similar to \cite{Yang20tro-teaser}).
 In some scenarios, using normals might boost performance, depending on the quality of the normals.
  More interestingly, using normals makes \robinstargnc~and \robingnc 
 attain similar performance: intuitively, better invariants lead to sparser \compatibility graphs, making
 the maximum \core{k} a better approximation of the maximum clique.
In terms of runtime, \robin can be up to 2 times faster than \robinstar.

\subsection{Registration Without Correspondences}\label{sec:exp:regNoCorrespondences}

\myParagraph{Setup} We test the extreme case where we 
assume all-to-all correspondences between the two point clouds~\cite{Yang20tro-teaser}.
We use the \teddyBear~model from \homebrewedDB \cite{Kaskman19-homebrewedDB}.
Contrary to \bunny, \teddyBear~comes with precomputed normals so we also test the point-with-normal variant of \robin.
We first downsample the \teddyBear~to 100 points to obtain the source point cloud.
Then, we apply a random rotation and translation to obtain the transformed point cloud.
To simulate a partial overlap, we randomly discard a percentage of the transformed point cloud.
For example, to simulate a 80\% overlap, we randomly discard 20\% of the transformed points.
We compare the performance of 6 algorithms in 40 Monte Carlo runs: (i-ii) \robinstargnc and \robingnc, which 
use points with normals;
(iii-iv) \teaserpp, \robinteaserpp, which disregard the normals,
(v) \GoICP~\cite{Yang16pami-goicp}, and (vi) ICP.
We feed \robinstargnc, \robingnc, \teaserpp,~and \robinteaserpp all possible correspondences: for each point in a source point cloud, we add all transformed points as potential correspondences.
For ICP, we use the identity as initial guess.

{\bf Results}.
Fig.~\ref{fig:3dreg-SPC} shows translation errors for different overlap ratios 
(rotation errors are similar and reported in~\suppMaterial).
ICP fails to converge in most cases, due to the inaccurate initial guess.
\GoICP~is robust until 20\% overlap.
\robinstargnc, \robingnc, \teaserpp ensure best performance for all overlap ratios. \robinstargnc~and \robingnc~have slightly lower errors (since they use extra information from the normals). 
As expected the variants using the maximum clique are more robust and still recover highly-accurate estimates at 5\% overlap. 
Unlike \robingnc, \robinteaserpp fails at all overlap ratios because the maximum $k$-core is only able to reject a small percentage of outliers, inducing failures in the robust estimator.
In terms of runtime, the \core{k} variants are faster: the average runtime for \robin ranges from $0.1$ms to $46$ms, and is 2 to 3.4 times faster than \robinstar.
  In our current implementation the construction of the \compatibility graph is single-threaded and dominates the total runtime; 
  this provides further speed-up opportunities.

\subsection{2D-3D Pose Estimation}
\label{sec:exp:cross-ratio}

\myParagraph{Setup} We evaluate \robin and \robinstar in a simulated 2D-3D pose estimation problem.
We use the cross ratio as a $4$-measurement invariant. 
We assume a pinhole camera projection model, with a image size of $640 \times 480$.
We general 100 collinear 3D points and ensure they are in the field of view of the camera (more details in~\suppMaterial).
We then project the 3D points back to the image plane to have 2D-3D correspondences.
Bounded random noise %
$\vepsilon_i \sim \calN(\bm{0},\sigma^2 \MI)$, with $\sigma=0.1$, is added to the 2D projections, 
and we ensure $\|\vepsilon_i\| \leq\! \beta  = 0.25 $.
Outliers are introduced by replacing some of the 2D points with randomly sampled points in the image.
We conduct 40 Monte Carlo runs %
and record: (i) the percentage of rejected outliers, and (ii) the percentage of inliers in the set returned by \robin and \robinstar.

\myParagraph{Results} Fig.~\ref{fig:cross-ratio-stats} shows 
that both \robin and \robinstar reject nearly 100\% outliers at all outlier levels. The percentages of inliers preserved by \robin and \robinstar are close to 100\% until 90\%, after which it drops to around 90\%; this is the case where outlier correspondences are starting to form large cliques.
We remark that \robin and \robinstar are able to reject outliers even in a setup (such as the one we used in our simulations) where the pose cannot be uniquely computed due to the collinearity of the points.

%% file: sections/fig-all-experiments-compressed.tex
\newcommand{\mpwthree}{5.9cm}
\newcommand{\mpwthreetwo}{11.4cm}
\newcommand{\myhspace}{\hspace{-6mm}}

\begin{figure*}[h]
	\begin{center}
	\begin{minipage}{\textwidth}
	\hspace{-0.2cm}
	\begin{tabular}{ccc}%
	  \vspace{2mm}
		\begin{minipage}{\mpwthree}%
			\centering%
			\includegraphics[width=\columnwidth]{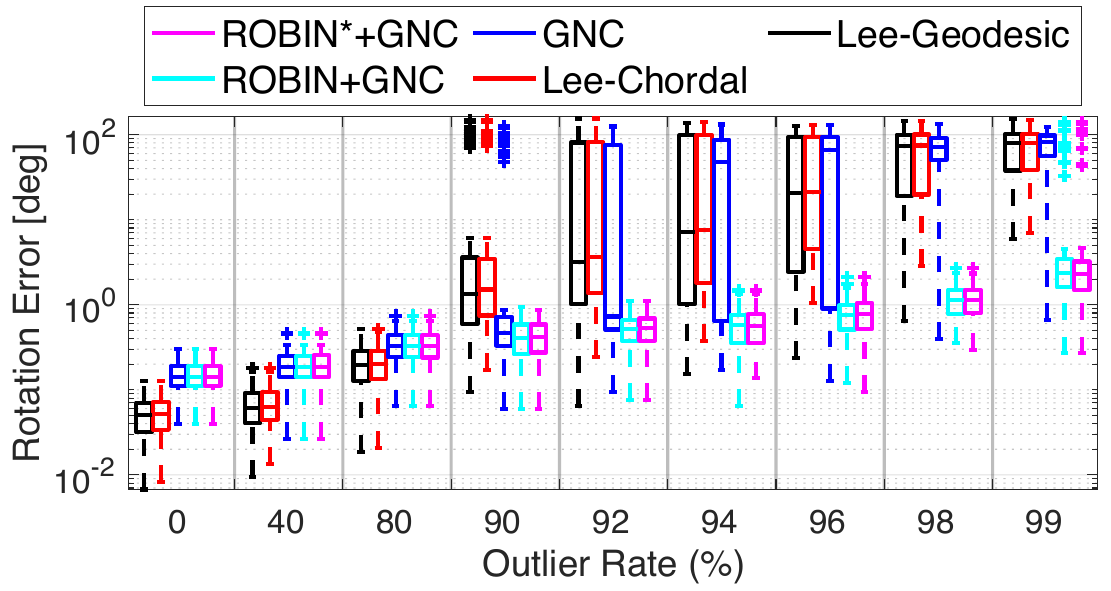} \\
			\vspace{-3mm}
			\caption{{Single rotation averaging.}}
			\label{fig:SRA-errors}
		  \end{minipage}
	  &
		\multicolumn{2}{c}{%
		\hspace{-5mm}
			\begin{minipage}{11.4cm}%

\input{sections/table-scanMatching}

			\end{minipage}
		}
		\\
		\begin{minipage}{\mpwthree}%
			\centering%
			\includegraphics[width=0.95\columnwidth]{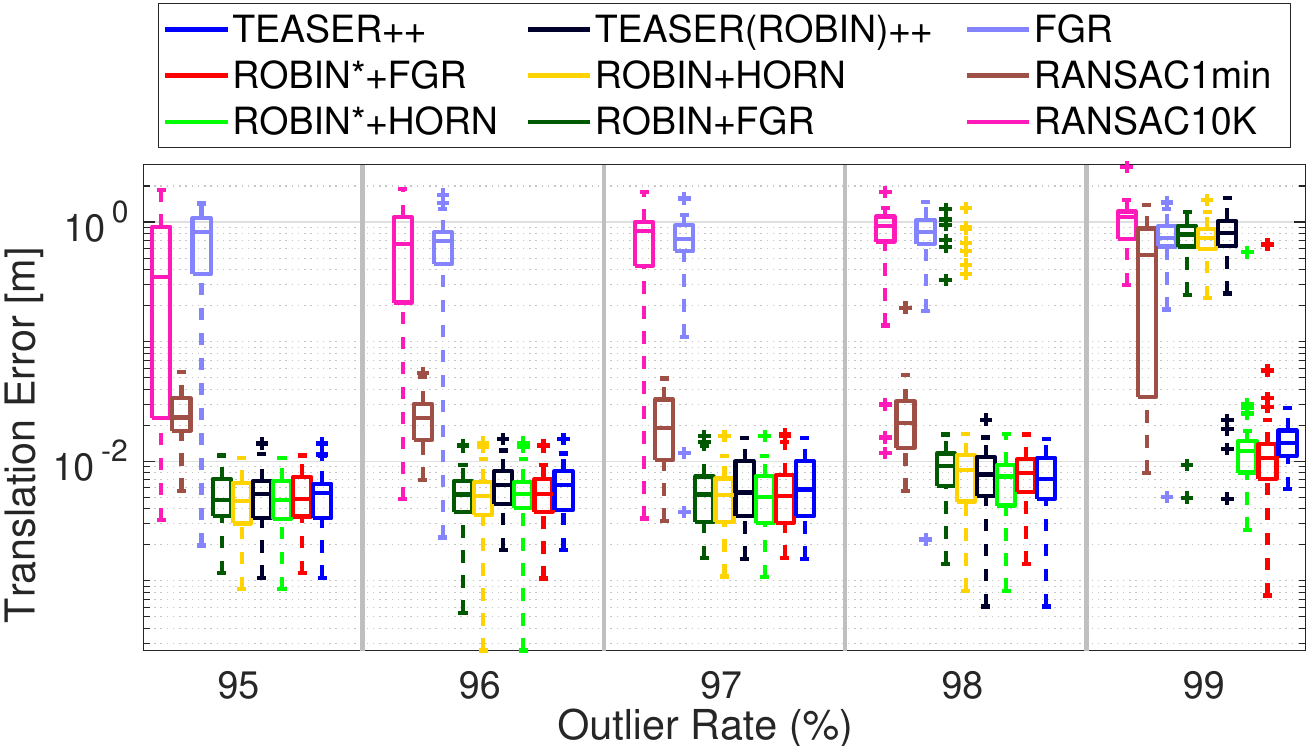} \\ \vspace{-7mm}\hspace{-5cm}
			\caption{{Registration with correspondences.}}\label{fig:3dreg-with-features}
			\end{minipage}
		& \hspace{-6mm}
			\begin{minipage}{\mpwthree}%
			\vspace{-3mm}
			\centering%
			\includegraphics[width=0.95\columnwidth]{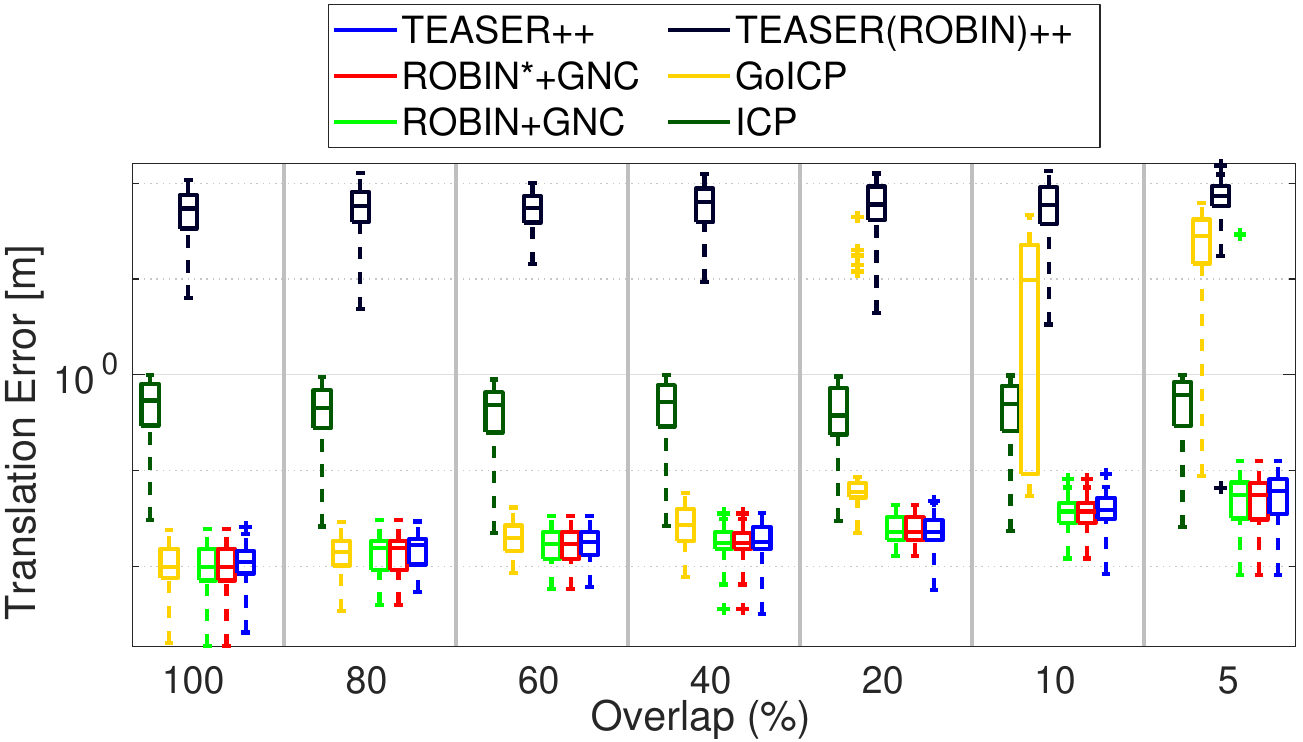}\\ \vspace{-7mm}\hspace{-5cm}
	        \caption{Registration without correspondences}\label{fig:3dreg-SPC}
			\end{minipage}
		& \hspace{-8mm}
			\begin{minipage}{\mpwthree}%
			\centering%
			\vspace{-3mm}
			\includegraphics[width=0.9\columnwidth]{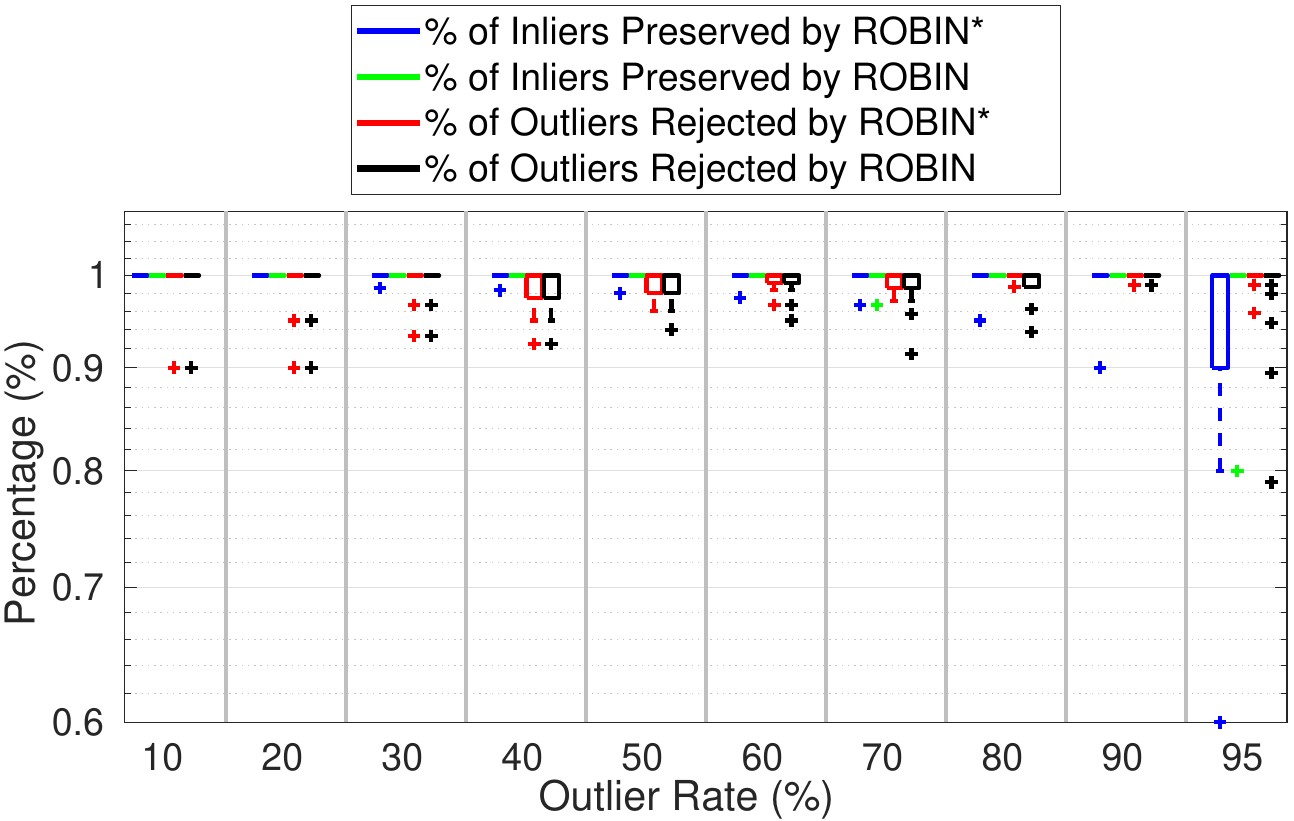}  \\ \vspace{-7mm}\hspace{-5cm}
	     \caption{2D-3D pose estimation}\label{fig:cross-ratio-stats}
			\end{minipage}\\
		\end{tabular}
	\end{minipage}
	\end{center}
\vspace{-7mm}
\end{figure*}

%% file: sections/table-scanMatching.tex
\renewcommand{\arraystretch}{1.4}

\adjustbox{max width=11.5cm}{%
\hspace{10mm}
\begin{tabular}{cccccccccc}
Scenes     & \shortstack{Kitchen \\(\%)} & \shortstack{Home 1 \\ (\%)}  & \shortstack{ Home 2\\ (\%)} & \shortstack{Hotel 1\\ (\%)} & \shortstack{Hotel 2 \\  (\%)} & \shortstack{ Hotel 3\\ (\%)} & \shortstack{ Study \\ Room (\%)} & \shortstack{ MIT Lab \\(\%)} & \shortstack{Average \\ Runtime [ms]} \\
\hline
\!\!\teaserpp \!\!& {\bf 93.9}         & {\bf 93.6}        & 71.6      & 96.0       & 87.5       & {\bf 92.6}       & {\bf 79.8}         & {\bf 76.6}   & 33.8                 \\
\hline
\!\!\robinteaserpp\!\! & 83.0         & 84.0        & 68.3      & 90.3       & 81.7       & 87.0       & 71.2          & 64.9   & 14.9                  \\
\hline
\!\!$\substack{\text{\normalsize\robinstargnc} \\ \text{\footnotesize(points with normals)}}$\!\! & 91.9         & 89.1        & {\bf 74.5}      & {\bf 96.5}       & {\bf 89.4}       & {\bf 92.6}       & 77.7          & 71.4   &  41.2                \\
\hline
\!\!$\substack{\text{\normalsize\robingnc} \\ \text{\footnotesize(points with normals)}}$\!\! & 91.9         & 89.1        & {\bf 74.5}      & {\bf 96.5}       & {\bf 89.4}       & {\bf 92.6}       & 77.7          & 71.4   & 39.5                   \\
\end{tabular}
}\\
\captionof{table}{\smaller Percentage of successful registrations and runtime on the 3DMatch dataset~\cite{Zeng17cvpr-3dmatch}.
\label{tab:scanMatching}}

%% file: sections/conclusion.tex
\section{Conclusion}
\label{sec:conclusion}

We proposed \robin, a fast approach to prune outliers in generic estimation problems.
\robin uses the notion of $n$-\invariant to quickly check if subsets of measurements 
are mutually compatible. %
Then, it prunes outliers by computing the maximum clique or maximum \core{k} of the graph induced by the 
invariants. We evaluate \robin in several problems and 
show that \robin boosts robustness of existing estimators --making them robust to more than 95\% outliers-- and
 runs in milliseconds on large-scale problems.

%% file: sections/app-proofInvariants.tex
\section{Proofs of Invariants}
\label{sec:proofInvariants}

\subsection{Pairwise Invariant for Rotation Averaging: Proof of Proposition~\ref{prop:rotAve-invariants}}
\begin{proof}
Applying $\vf$ to both sides of~\eqref{eq:rot-avg-gen-model} for two measurements $i$ and $j$, we get
\begin{multline} 
\vf(\MR_i, \MR_j) = \vf(\MR \Exp(\vepsilon_i), \MR \Exp(\vepsilon_j)) = \\
  \Exp(\vepsilon_i)\tran \MR\tran \MR \Exp(\vepsilon_j) = \Exp(\vepsilon_i)\tran
 \Exp(\vepsilon_j)
 \end{multline} 
 which proves eq.~\eqref{eq:rotAve-invariants} (since $\newnoise = \vepsilon$) and is independent on $\MR$, hence proving 
 pairwise invariance. 
\end{proof}

\subsection{Pairwise Invariant for Point Cloud Registration: Proof of Proposition~\ref{prop:p-reg-invariants}}
\begin{proof}
By inspection, applying $\vf$ to both sides of~\eqref{eq:p-reg-gen-model} for two measurements $i$ and $j$, we get 
$\vf(\vb_i, \vb_j) = \vf( \MR \va_i + \vt + \vepsilon_i, \MR \va_j + \vt + \vepsilon_j) = 
\| \MR \va_j + \vepsilon_j - \MR \va_i + \vepsilon_i\| $.
Since the 2-norm $\|\cdot\|$ is invariant to rotation and recalling the assumption of isotropic noise:
$\| \va_j  - \va_i + \MR \tran \vepsilon_j - \MR \tran \vepsilon_i\| = \| \va_j  - \va_i + \vdelta_j - \vdelta_i\|$, which proves~\eqref{eq:p-reg-invariants} and is independent on $\MR$ and $\vt$, hence proving 
 pairwise invariance. 
 \end{proof}

\subsection{Pairwise Invariant for Point-with-Normal Registration: Proof of Proposition~\ref{prop:pwithn-reg-invariants}}
\begin{proof}
In Proposition~\ref{prop:p-reg-invariants} we have already shown that  $\| \vb_j - \vb_i \|$ is an invariant 
for the point measurements.
Now we show that $(\vn^b_i)\tran \vn^b_j$ is also an invariant for the normal measurements.
Applying $(\vn^b_i)\tran \vn^b_j$ to both sides of the second equation in~\eqref{eq:pwithn-reg-gen-model} for two measurements $i$ and $j$, we get 
\begin{multline}
(\vn^b_i)\tran \vn^b_j = (\MR \Exp(\vnu_i) \vn^a_i)\tran (\MR \Exp(\vnu_j) \vn^a_j) = \\
(\vn^a_i)\tran  \Exp(\vnu_i)\tran \Exp(\vnu_j) \vn^a_j,
\end{multline}
which does not depend on $\MR$ and $\vt$.
\end{proof}

%% file: sections/app-invariantsInVision.tex
\section{Invariants in Computer Vision}
\label{sec:invariantTheory}

The interested reader can find a broader discussion on invariance in~\cite{Mundy92book}, which provideds a detailed survey of early applications of invariants for object recognition.
In addition, when the transformation under estimation is an element of a Lie Group,~\cite{Gros92-projectiveInvariantsTheory} provides theoretical results on the number of invariants available.

There are much discussion on what consititutes as ``good'' invariants. 
One of the metrics is noise resistancy: how resistant a particular invariant is to the introduction of noise. 
For example, \cite{Mundy92book} provides results on noise-resistant invariants of curves.
For our purpose, we think the number of measurements used for calculating invariants is also a good metric.
For \robin, given an $n$-invariant, our outlier rejection approach checks all $\nchoosek{N}{n}$ subsets of $n$ measurements. Hence pairwise invariants (and in general, invariants for small $n$) are more desirable than one with large $n$.

The complement of the compatibility graph used in \robin can also be used for outlier rejection.
In that case, instead of finding the maximum clique, the solver aims to find the minimal vertex cover for the graph.
Such approach is used in \cite{Enqvist09iccv} for 3D-3D and 2D-3D correspondence outlier rejection.

While this paper focuses on computing a single estimate of $\vxx$ from the measurements in the maximum clique 
(or the maximum \core{k}) of the compatibility graph, our graph-theoretic framework provides a natural way to design a \emph{multi-hypothesis estimator}, where one computes an estimate for each maximal clique of size larger than the expected number of inliers. Intuitively, each  maximal clique describes a set of mutually \compatible measurements and will lead to a different estimate of $\vxx$.

%% file: sections/app-compatibilityTests.tex
\section{\Compatibility Tests}
\label{sec:compabilityTests}

\subsection{\Compatibility test for single rotation averaging}
 Eq.~\eqref{eq:rotAve-invariants} states that any pair of measurements satisfies:
 \beq
 \label{eq:test-rotAve-ref1}
\vf(\MR_i, \MR_j) \doteq \MR_i\tran \MR_j = \Exp(\newnoise_{i})\tran
 \Exp(\newnoise_{j}) \doteq \Exp(\newnoise_{ij}).
 \eeq
 Therefore, if both measurements $\MR_i,\MR_j$ are inliers, then both $\newnoise_{i}$ and $\newnoise_{j}$ satisfy $\| \newnoise_{i} \| \leq \beta$ and $\| \newnoise_{j} \| \leq \beta$. Intuitively, this will create an upper bound for $\norm{\newnoise_{ij}}$, violating which indicates there are at least one outlier in the two measurements $\MR_i$ and $\MR_j$. 
 Since these are rotations, $\norm{\newnoise}$ is the geodesic distance~\cite{Hartley13ijcv} from the identity rotation, which is simply the rotation angle associated to a rotation. To make this choice explicit, for a rotation $\Exp(\newnoise)$, we define:
 \beq
 \| \newnoise \| \doteq \dist{\Exp(\newnoise)}{\eye_3}
 \eeq
where $\dist{\cdot}{\cdot}$ denotes the geodesic distance between two 3D rotations.
\HYedit{We claim that if two measurements $\MR_i,\MR_j$ satisfy~\eqref{eq:test-rotAve-ref1} for some  
$\newnoise_i,\newnoise_j$ such that $\| \newnoise_i \|\leq \beta$ and $\| \newnoise_j \|\leq \beta$,
then: 
 \beq
  \label{eq:test-rotAve}
  \|\newnoise_{ij}\| \leq 2\beta.
 \eeq

\begin{proof} Using eq.~\eqref{eq:test-rotAve-ref1} and the triangle inequality of geodesic distance in $\SOthree$, we have:
\bea
\norm{\newnoise_{ij}} = \dist{\Exp(\newnoise_i)}{\Exp(\newnoise_j)} \\
\leq \dist{\Exp(\newnoise_i)}{\eye_3} + \dist{\Exp(\newnoise_j)}{\eye_3} \\
= \norm{\newnoise_i} + \norm{\newnoise_j} \leq 2\beta,
\eea
proving eq.~\eqref{eq:test-rotAve}.
\end{proof}
}

\subsection{\Compatibility test for point-with-normal registration}

Eq.~\eqref{eq:pwithn-reg-invariants} states that any pair of measurements satisfies:
  \beq
\label{eq:pwithn-reg-invariants-ref1}
\hspace{-4mm} \vf(\vy_{i}, \vy_{j})
\doteq \matTwo{  \| \vb_j - \vb_i \| \\ (\vn^b_i)\tran \vn^b_j }
= \matTwo{  \| \va_j  - \va_i + \newnoise_j^p - \newnoise_i^p\| 
 \\ 
(\vn^a_i)\tran  \Exp(\newnoise_i^n)\tran \Exp(\newnoise_j^n) \vn^a_j }
\eeq 
 Therefore, any pair of inliers must satisfy~\eqref{eq:pwithn-reg-invariants-ref1} for some choice of
 $\newnoise_i^p, \newnoise_j^p,\newnoise_i^n,\newnoise_j^n$ such that $\norm{\newnoise_i^p} \leq \beta$, $\norm{\newnoise_j^p}\leq \beta$, $\norm{\newnoise_i^n} \leq \beta$, $\norm{\newnoise_j^n} \leq  \beta$.
 Again, we look for a simpler condition that serves as a necessary condition to be satisfied by a pair of inliers.
In particular, we claim that a pair of inliers is always such that:
\bea
-2\beta \leq \| \vb_j - \vb_i \| - \| \va_j  - \va_i \| \leq 2\beta, \label{eq:condPoints}
 \\
| \arccos\left( (\vn^b_i)\tran \vn^b_j \right) - \arccos\left( (\vn^a_i)\tran \vn^a_j \right)  | {\leq} 2\beta. \label{eq:condNormals}
\eea 

\begin{proof}
We have already shown that inlier points must satisfy~\eqref{eq:condPoints} in Section~\ref{sec:compatibilityTests} (see eq.~\eqref{eq:test-p-ref3}). 
We now show that a pair of inliers must also satisfy~\eqref{eq:condNormals}. 
\HYedit{
To do so, let the geodesic distance between two vectors $\vn_i,\vn_j$ on the unit $2$-sphere $\usphere{2}$ be:
\bea
\distsphere{\vn_i}{\vn_j} \doteq \arccos(\vn_i\tran \vn_j).
\eea
Then we have:
\bea
\arccos( (\vn^b_i)\tran \vn^b_j ) = \distsphere{\vn^b_i}{\vn^b_j} \\
\overset{\text{\small triangle inequality}}{\leq} \distsphere{\vn^b_i}{\vn^a_i} + \distsphere{\vn^a_i}{\vn^b_j} \\
= \distsphere{\Exp(\newnoise_i^n) \vn^a_i}{\vn^a_i} + \distsphere{\vn^a_i}{\vn^b_j} \\
\leq \beta + \distsphere{\vn^a_i}{\vn^b_j} \label{eq:invokerotdistance1}\\
 \hspace{-4mm} \overset{\text{\small triangle inequality}}{\leq} \beta + \distsphere{\vn^a_i}{\vn^a_j} + \distsphere{\vn^a_j}{\vn^b_j} \\
= \beta + \distsphere{\vn^a_i}{\vn^a_j} + \distsphere{\vn^a_j}{\Exp(\newnoise_j^n) \vn_j^a} \\
\leq \beta + \distsphere{\vn^a_i}{\vn^a_j} + \beta \label{eq:invokerotdistance2} \\
= \arccos((\vn^a_i)\tran \vn^a_j) + 2 \beta,
\eea
proving one direction of eq.~\eqref{eq:condNormals}. In eq.~\eqref{eq:invokerotdistance1} and~\eqref{eq:invokerotdistance2}, we have used the fact that, for any $\vn \in \usphere{2}$, if a 3D rotation $\Exp(\newnoise)$ satisfies $\norm{\newnoise} \leq \beta$, then $\distsphere{\vn}{\Exp(\newnoise)\vn} \leq \beta$. To see this, let us write $\Exp(\newnoise)$ using the Rodrigues' formula:
\bea
\MR = \cos(\theta)\eye_3 + \sin(\theta) \hatmap{\vu} + (1-\cos(\theta))\vu \vu\tran,
\eea 
where $\vu \in \usphere{2}$ is the rotation axis, and $\theta \in \Real{}$ is the rotation angle around $\vu$, such that $\newnoise = \theta \vu$. Note that since $\norm{\newnoise} \leq \beta$, we have $\abs{\theta} \leq \beta$. Then we lower bound the dot product between $\vn$ and $\MR\vn$:
\begin{multline}
\vn\tran \MR \vn = \cos(\theta) + (1-\cos(\theta))(\vn\tran \vu)^2 \geq \\
\cos(\theta) \geq \cos(\beta),
\end{multline}
which in turns yields:
\begin{multline}
\distsphere{\vn}{\MR\vn} = \arccos(\vn\tran \MR \vn) \leq \\
\arccos(\cos(\beta)) = \beta.
\end{multline}

The other direction can be done by reversing the order of the proof,~\ie~starting from $\distsphere{\vn^a_i}{\vn^a_j}$ and invoking the triangle inequality.
}
\end{proof}

In the following, we rewrite~\eqref{eq:condNormals} in a more convenient form that does not involve
trigonometric functions and is faster  to compute.
Recalling from basic trigonometry that for two variables $-1\leq c_a,c_b \leq +1$ it holds~\cite[Eq. 4.4.33]{Abramowitz74book}:
\beq
\arccos(c_b) - \arccos(c_a) = \arccos( c_a c_b + \sqrt{ (1\!-\!c_a^2) (1\!-\!c_b^2)}  ) \nonumber
\eeq
Using this relation in~\eqref{eq:condNormals} and denoting $c_a \doteq (\vn^a_i)\tran \vn^a_j$ and 
$c_b \doteq (\vn^b_i)\tran \vn^b_j$:
\bea
 -2\beta \leq \arccos\left( c_b \right) - \arccos\left( c_a \right)   \leq 2\beta \iff  \nonumber\\
  -2\beta \leq \arccos\left( c_a c_b + \sqrt{ (1\!-\!c_a^2) (1\!-\!c_b^2)}  \right) 
  \leq 2\beta \iff  \nonumber\\
  \myGrayMath{since $\arccos$ conventionally returns an angle in $[0,\pi]$}  \nonumber\\
\arccos\left( c_a c_b + \sqrt{ (1\!-\!c_a^2) (1\!-\!c_b^2)}  \right) 
  \leq 2\beta \iff \nonumber\\
   \myGrayMath{applying the decreasing function $\cos$ to both sides}  \nonumber\\
  c_a c_b + \sqrt{ (1\!-\!c_a^2) (1\!-\!c_b^2)} 
  \geq \cos(2\beta) \iff \nonumber\\
   \sqrt{ (1\!-\!c_a^2) (1\!-\!c_b^2)} 
  \geq \cos(2\beta) - c_a c_b \iff \nonumber\\
  \myGrayMath{squaring both sides}  \nonumber\\
  (1\!-\!c_a^2) (1\!-\!c_b^2)
  \geq \cos(2\beta)^2 + c_a^2 c_b^2 \!-\! 2 \cos(2\beta)c_ac_b\iff \nonumber\\
    1 + c_a^2c_b^2 \!-\! c_a^2 \!-\! c_b^2
  \geq \cos(2\beta)^2 + c_a^2 c_b^2 - 2 \cos(2\beta)c_ac_b\iff \nonumber\\
  2 \cos(2\beta)c_ac_b + 1 - c_a^2 - c_b^2 \nonumber
  \geq \cos(2\beta)^2 %
\eea
 which is much faster to evaluate since it avoids computing a trigonometric function for each evaluation ($\cos(2\beta)$ has to be computed just once). In our original notation:
 \bea
 2 \cos(2\beta) (\vn^a_i)\tran \vn^a_j (\vn^b_i)\tran \vn^b_j + \nonumber\\ 
 + 1 - ((\vn^a_i)\tran \vn^a_j)^2 - ((\vn^b_i)\tran \vn^b_j)^2   \geq \cos(2\beta)^2.
 \eea

\subsection{\Compatibility test for 2D-3D camera pose estimation}
First, using the 3D points $\vp_i,i=1,2,3,4$, we compute the cross ratio $\crossratio$ according to eq.~\eqref{eq:cross-invariants} as:
\bea
\crossratio = \frac{ \| (\vp_1)^\vee - (\vp_2)^\vee \| \| (\vp_3)^\vee - (\vp_4)^\vee \|  }{ 
\| (\vp_1)^\vee - (\vp_3)^\vee \| \| (\vp_2)^\vee - (\vp_4)^\vee \|  }.
\eea
Next, according to the invariance of the cross ratio, we have:
\bea
\crossratio = \frac{\| {\vy}_{1} - \newnoise_1 - {\vy}_{2} + \newnoise_2 \| \| {\vy}_{3} - \newnoise_3 - {\vy}_{4} + \newnoise_4 \|}{ 
\| {\vy}_{1} - \newnoise_1 - {\vy}_{3} + \newnoise_3 \| \| {\vy}_{2} - \newnoise_2 - {\vy}_{4} + \newnoise_4\| }, \label{eq:crossratio-fraction}
\eea
from which we try to lower and upper bound the right-hand side to get inequalities that serve to test the compatibility of 2D-3D correspondences. To do so, we use the triangle inequality to write:
\bea
\norm{\vy_{12}} - 2\beta \leq \norm{\vy_1 - \newnoise_1 - \vy_2 + \newnoise_2} \leq \norm{\vy_{12}} + 2\beta, \\
\norm{\vy_{34}} - 2\beta \leq \norm{\vy_3 - \newnoise_3 - \vy_4 + \newnoise_4} \leq \norm{\vy_{34}} + 2\beta,\\
\norm{\vy_{13}} - 2\beta \leq \norm{\vy_1 - \newnoise_1 - \vy_3 + \newnoise_3} \leq \norm{\vy_{13}} + 2\beta,\\
\norm{\vy_{24}} - 2\beta \leq \norm{\vy_2 - \newnoise_2 - \vy_4 + \newnoise_4} \leq \norm{\vy_{24}} + 2\beta,
\eea
where $\norm{\vy_{ij}} \doteq \norm{\vy_i - \vy_j}$ denotes the distance between $\vy_i$ and $\vy_j$. From the inequalities above, we obtain:
\begin{multline}
\frac{\parentheses{\norm{\vy_{12}} - 2\beta}\parentheses{\norm{\vy_{34}}-2\beta}}{ \parentheses{\norm{\vy_{13}} + 2\beta} \parentheses{\norm{\vy_{24}} + 2\beta} } < \tau \\ <
\frac{\parentheses{\norm{\vy_{12}} + 2\beta}\parentheses{\norm{\vy_{34}} + 2\beta}}{ \parentheses{\norm{\vy_{13}} - 2\beta} \parentheses{\norm{\vy_{24}} - 2\beta} }, \label{eq:crossratiotest}
\end{multline} 
which is the compatibility test for cross ratio. Note that strict inequalities hold in eq.~\eqref{eq:crossratiotest} because the maximum (resp. minimum) of the numerator and the minimum (resp. maximum) of the denominator of the right-hand side in eq.~\eqref{eq:crossratio-fraction} cannot be simultaneously attained.

%% file: sections/app-proofGraph.tex
\section{Proofs of Compatibility Graph Inlier Structures}
\label{sec:proofGraph}

\subsection{Inliers and Maximum Clique: Proof of Theorem~\ref{thm:inliers-form-clique}}
\begin{proof} 
Our \compatibility tests are designed to pass as long as the subset under test includes all inliers.
Therefore, the inliers will form a complete subgraph (clique) of $\calG$ and will certainly be part of a \emph{maximal clique} of size at least $n$. 
\end{proof}

%% file: sections/app-hornWithNormals.tex
\newcommand{\dimension}{d}
\newcommand{\bcenter}{\bar{\vb}}
\newcommand{\acenter}{\bar{\va}}
\newcommand{\aref}{\tilde{\va}}
\newcommand{\bref}{\tilde{\vb}}
\renewcommand{\det}[1]{\text{det}\parentheses{#1}}
\newcommand{\orthogonal}[1]{\text{O}\parentheses{#1}}
\newcommand{\tls}{\scenario{TLS}}
\newcommand{\gnc}{\scenario{GNC}}

\section{\GNC for Point-with-Normal Registration}
\label{sec:hornWithNormals}

According to~\cite{Yang20ral-GNC}, in order to use \GNC for robust point-with-normal registration, we need to design a \emph{non-minimal solver} that can solve the point-with-normal registration without outliers. Fortunately, as we will show below, the point-with-normal registration problem admits a closed-form solution if assuming all correspondences are inliers.

\subsection{Outlier-free Closed-form Solution}
Our result holds true for all $d$-dimensional ($d\geq 2$) point-with-normal registration, with $d=3$ being a special example.
\begin{theorem}[Point-with-Normal Registration]\label{thm:pointwithnormalclosedform}
Let $\calA = \cbrace{\parentheses{\va_i,\vm_i}}_{i=1}^N$ and $\calB = \cbrace{\parentheses{\vb_i,\vn_i}}_{i=1}^N$ be two point clouds with estimated unit normals, where $\va_i, \vb_i \in \Real{\dimension}$ are the points in $\dimension$-dimensional Euclidean space and $\vm_i,\vn_i \in \usphere{\dimension -1}$ are the unit normals. Consider finding the best rigid transformation to align $\calA$ and $\calB$ by solving the following optimization problem:
\bea 
\min_{\substack{\MR \in \SO{\dimension}, \\ \vt \in \Real{\dimension}}}  \sum_{i=1}^N \frac{1}{\alpha_i^2} \norm{\vb_i - \MR \va_i - \vt}^2 + \frac{\rho_i}{\beta_i^2} \norm{\vn_i - \MR \vm_i}^2, \label{eq:outlier-free-problem}
\eea
where $\alpha_i > 0$ and $\beta_i > 0$ are normalizing constants that are proportional to the magnitude of the noise, and $\rho_i$ determines the relative weight between point-to-point distance and normal-to-normal distance, 
then the globally optimal solution $(\MR^\star, \vt^\star)$ can be computed as:
\bea
\begin{cases}
\MR^\star = \MU \diag{\bracket{1,\dots,1,\det{\MU}\det{\MV}}} \MV\tran \\
\vt^\star = \bcenter - \MR^\star \acenter
\end{cases},
\eea
where $\acenter$ and $\bcenter$ are the weighted center of $\calA$ and $\calB$:
\bea
\acenter \doteq \frac{\sum_{i=1}^N \eta_i \va_i}{\sum_{i=1}^N \eta_i} , \quad \bcenter \doteq \frac{\sum_{i=1}^N \eta_i \vb_i}{\sum_{i=1}^N \eta_i}, \label{eq:aandbcenter}
\eea
and $\MU, \MV \in \orthogonal{\dimension}$ come from the following singular value decomposition:
\bea
\MM = \sum_{i=1}^N \eta_i \bref_i \aref_i\tran + \sum_{i=1}^N \kappa_i \vn_i \vm_i\tran = \MU \MS \MV\tran, \label{eq:svdM}
\eea
assuming that the singular values in $\MS$ are ordered in descending order. Note that in eq.~\eqref{eq:aandbcenter} and~\eqref{eq:svdM}, we use the following notations:
\bea
\eta_i \doteq \frac{1}{\alpha_i^2},\ \kappa_i \doteq \frac{\rho_i}{\beta_i^2},\ \aref_i \doteq \va_i - \acenter,\ \bref_i \doteq \vb_i - \bcenter,
\eea
for any $i=1,\dots,N$.
\end{theorem}
\begin{proof}
We start by taking the derivate of the objective function of~\eqref{eq:outlier-free-problem}, denoted as $f(\MR,\vt)$, \wrt $\vt$, and set it to zero:
\bea 
\frac{\partial f}{\partial \vt} = 2\sum_{i=1}^N \eta_i \parentheses{\vt + \MR\va_i - \vb_i} = \zero,
\eea 
which gives the closed-form solution of $\vt^\star$ in terms of $\MR^\star$:
\bea 
\vt^\star = \bcenter - \MR^\star \acenter. \label{eq:tinR}
\eea
Now we insert eq.~\eqref{eq:tinR} back to problem~\eqref{eq:outlier-free-problem} and obtain a problem that only seeks to find the best rotation:
\bea
\min_{\MR \in \SO{\dimension}} \sum_{i=1}^N \eta_i \norm{\bref_i - \MR \aref_i}^2 + \kappa_i \norm{\vn_i - \MR \vm_i}^2, \label{eq:rotation-only}
\eea
Problem~\eqref{eq:rotation-only} is in the form of a Wahba problem, and hence admits a closed form solution from the singular value decomposition~\eqref{eq:svdM} (see~\cite{markley1988jas-svdAttitudeDeter} for details).
\end{proof}

\subsection{Robust Registration}
In the presence of outliers, we will modify the cost function in~\eqref{eq:outlier-free-problem} to be the truncated least squares (\tls) cost function to gain robustness. To do so,
denote:
\bea
r_i^2\parentheses{\MR,\vt} = \frac{1}{\alpha_i^2} \norm{\vb_i - \MR \va_i - \vt}^2 + \frac{\rho_i}{\beta_i^2} \norm{\vn_i - \MR \vm_i}^2,
\eea 
then robust estimation seeks to solve the following \tls estimation problem in the presence of outliers:
\bea
\MR^\star, \vt^\star = \argmin_{\MR \in \SO{\dimension},\vt \in \Real{\dimension}} \sum_{i=1}^N \min \cbrace{\frac{r_i^2\parentheses{\MR,\vt}}{\zeta_i^2},\barcsq}, \label{eq:tls-robust-registration}
\eea
where $\zeta_i > 0$ is the maximum allowed residual for $r_i \parentheses{\MR,\vt}$ to be considered as an inlier, $\barc$ (default $\barc=1$) is a tuning constant (see~\cite{Yang20tro-teaser,Yang20ral-GNC,Yang20nips-certifiablePerception} for details).  

To solve problem~\eqref{eq:tls-robust-registration}, we use the \gnc meta-heuristics proposed in~\cite{Yang20ral-GNC}, which requires a non-minimal solver to solve the following weighted least squares problem:
\bea
\min_{\MR\in\SO{\dimension},\vt \in \Real{\dimension}} \sum_{i=1}^N \frac{w_i}{\zeta_i^2} r_i^2\parentheses{\MR,\vt}, \label{eq:gncweightedleastsquares}
\eea
which is equivalent to:
\bea
\min_{\substack{\MR\in\SO{\dimension}, \\ \vt \in \Real{\dimension}}} \sum_{i=1}^N \eta_i' \norm{\vb_i - \MR\va_i - \vt}^2 + \kappa_i' \norm{\vn_i - \MR\vm_i}^2,
\eea
by letting $\eta_i' = \frac{w_i}{\alpha_i^2 \zeta_i^2}$ and $\kappa_i' = \frac{w_i \rho_i}{\beta_i^2 \zeta_i^2}$. Hence, the closed-from solution in Theorem~\ref{thm:pointwithnormalclosedform} applies. The \gnc heuristics then alternates between solving the weighted least squares problem~\eqref{eq:gncweightedleastsquares} and updating the weights until convergence. We refer the interested reader to~\cite{Yang20ral-GNC} for implementation details.

%% file: sections/app-experiments.tex
\section{Extra Experimental Results}
\label{sec:app-exp-results}

\subsection{Single Rotation Averaging}
Fig.~\ref{fig:app-sra-full}(a) shows the performance of the five algorithms in Section~\ref{sec:exp:singlerotationavg} under the full set of outlier rates from $0\%$ to $99\%$ (under 100 Monte Carlo runs). We see that \robinstargnc and \robingnc consistently returns accurate estimates up to $98\%$ outlier rate, while the other three algorithms start failing at $90\%$ outliers. In addition, Fig.~\ref{fig:app-sra-full}(b)-(c) boxplot the percentage of inliers preserved, the percentage of outliers rejected, and the inlier rate within the maximum clique (\robinstar), and within the maximum $k$-core (\robin), respectively. We see that both maximum clique and maximum $k$-core preserves almost all of the inliers and rejects almost all of the outliers (typically over $90\%$), hence, finding a much refined set of measurements that has a high inlier rate (typically above $60\%$ when the outlier rate is up to $98\%$). As a result, \GNC is able to return an accurate estimate.

\input{sections/fig-app-SRA-full}

\subsection{Registration With Correspondences}
{\bf Simulated Benchmarks}
Fig.~\ref{fig:app-PCREG-full}(a) shows an example registration problem used for our tests with the \bunny model. 
Fig.~\ref{fig:app-PCREG-full}(b) and (c) show the rotation and translation estimation errors of the methods tested under outlier rates form 0\% to 90\%.
Fig.~\ref{fig:app-PCREG-full}(e) and (f) show the rotation and translation estimation errors of the methods tested under outlier rates form 95\% to 99\%.
Note that for the tests between 0\% and 90\% outlier rates, 100 correspondences are used. For the tests between 95\% and 99\% outlier rates, 1000 correspondences are used.
Note that \teaserpp~is equivalent to \robinstar~with a decoupled rotation and translation solver.
\robin and \robinstar enabled methods dominate all other methods.
Noticeably, with \robinstar, Horn's method is robust up until 98\% outlier rates.

In addition, Fig.~\ref{fig:app-PCREG-full}(d) and (g) shows the timing data of different algorithms we tested. 
All \robin and \robinstar enabled methods complete execution at the order of tens of milliseconds, achieving real-time performance.
Noticeably, from 95\% to 99\% outlier rates, \robinteaserpp has an average runtime of about 4 milliseconds, which is about 3 times faster than \teaserpp.

\input{sections/fig-app-PCREG-full}

{\bf 3DMatch Dataset}
Fig.~\ref{fig:app-3dmatch-cases} show some of the successful as well as failed cases from our tests on the 3DMatch dataset using \robinstargnc.

\input{sections/fig-app-3dmatch}

\subsection{Registration Without Correspondences}
Fig.~\ref{fig:app-SPC-full}(a) shows the \teddyBear model used in our tests for registration problems without correspondences.
Fig.~\ref{fig:app-SPC-full}(b) and (c) shows the rotation and translation estimation errors of the methods tested under different overlap rates.
Fig.~\ref{fig:app-SPC-full}(d) shows the timing spent by \robin and \robinstar on finding maximum $k$-core and maximum cliques respectively under different overlap rates.

\input{sections/fig-app-SPC-full}

\subsection{2D-3D Pose Estimation}

We evaluate \robin~on a simulated 2D-3D outlier rejection problem.
We use the cross ratio as a $4$-measurement invariant, and use \eqref{eq:crossratiotest} for the compatibility test. 
We assume a pinhole camera projection model, with a image size of $640 \times 480$.
We randomly sample two points on the image plane, and project them to 3D points in the camera coordinate frame with random depths.
We then generate a random extrinsic transformation to transform the 3D points to the world frame.
Using the two 3D points in the world frame, we randomly generate 100 points in the segment connecting the two points.
We then project the 3D points back to image plane to have 2D-3D correspondences.
Bounded random noises (as shown in~\eqref{eq:cross-invariants}) $\vepsilon_i \sim \calN(\bm{0},\sigma^2 \MI), \sigma=0.1$ are introduced to the 2D points.
$\vepsilon_i$ are sampled until the resulting vector satisfies $\|\vepsilon_i\| \leq\!\!\beta_i$ where $\beta_{i}=0.25$.
Outliers are introduced to the 2D points by randomly sampling points on the image plane.
We conduct 40 Monte Carlo runs with the randomly generated 2D-3D correspondences, and record two statistics: (i) the percentage of outliers rejected, (ii) the percentage of inliers among all inliers in the graph structures identified by \robin and \robinstar.

\input{sections/fig-app-CR-full}

%% file: sections/fig-app-SRA-full.tex
\begin{figure*}
\begin{minipage}{\textwidth}
\begin{tabular}{c}%
\begin{minipage}{\textwidth}%
			\centering%
			\includegraphics[width=\columnwidth]{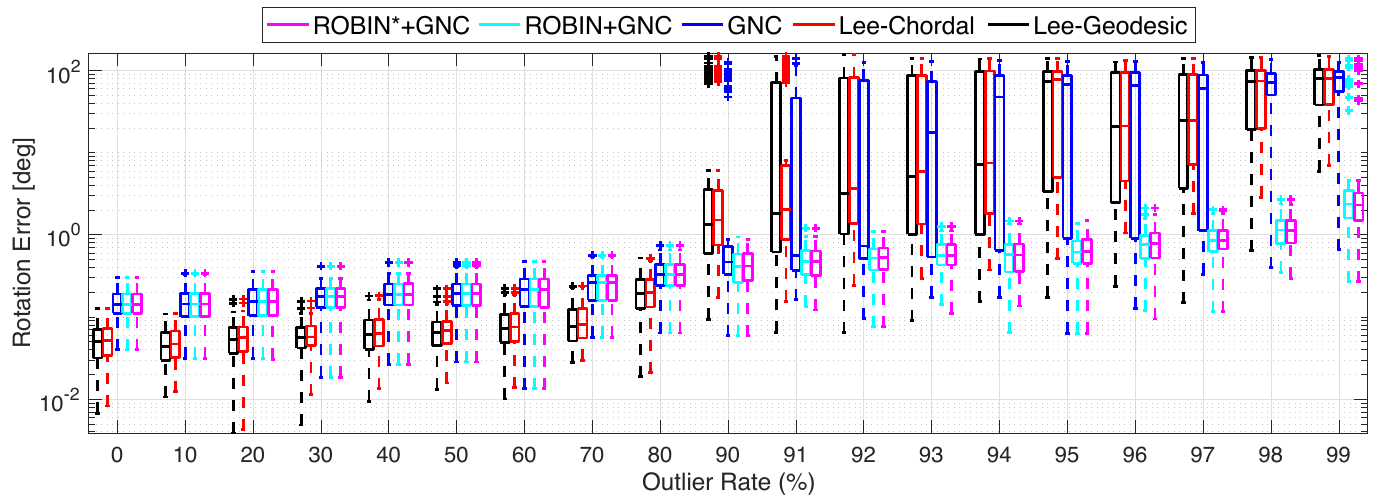} \\
 \vspace{-2mm} (a) Rotation estimation error (compared to the groundtruth) \wrt increasing outlier rates.
\end{minipage}
\vspace{1mm}
\\
\begin{minipage}{\textwidth}%
			\centering%
			\includegraphics[width=\columnwidth]{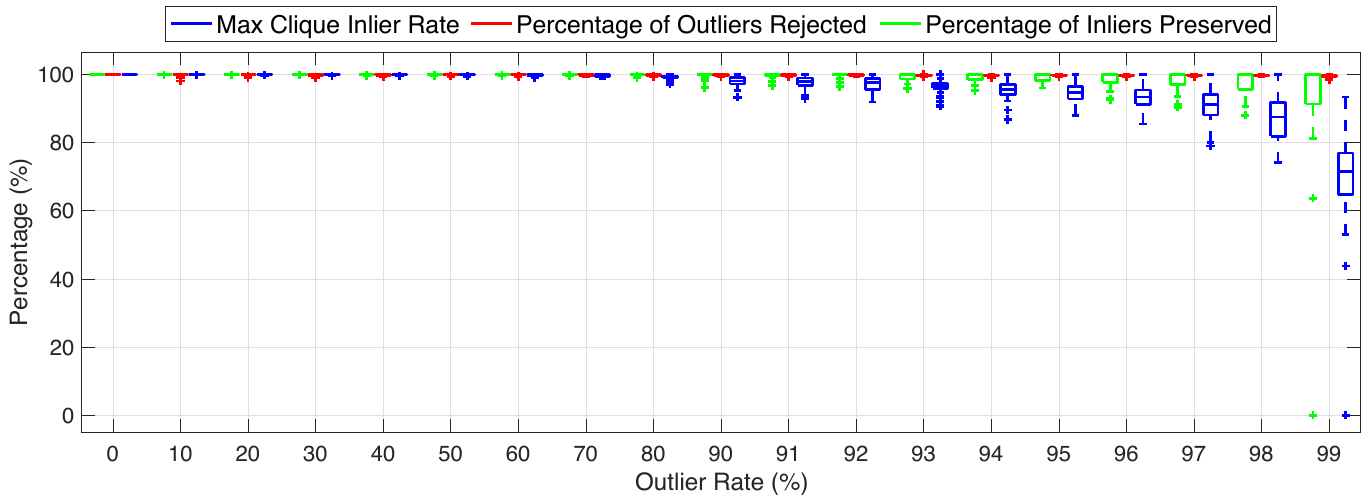} \\
 \vspace{-2mm} (b) Percentage of inliers preserved, percentage of outliers rejected, and inlier rate within the maximum clique.
\end{minipage}
\vspace{1mm}
\\
\begin{minipage}{\textwidth}%
			\centering%
			\includegraphics[width=\columnwidth]{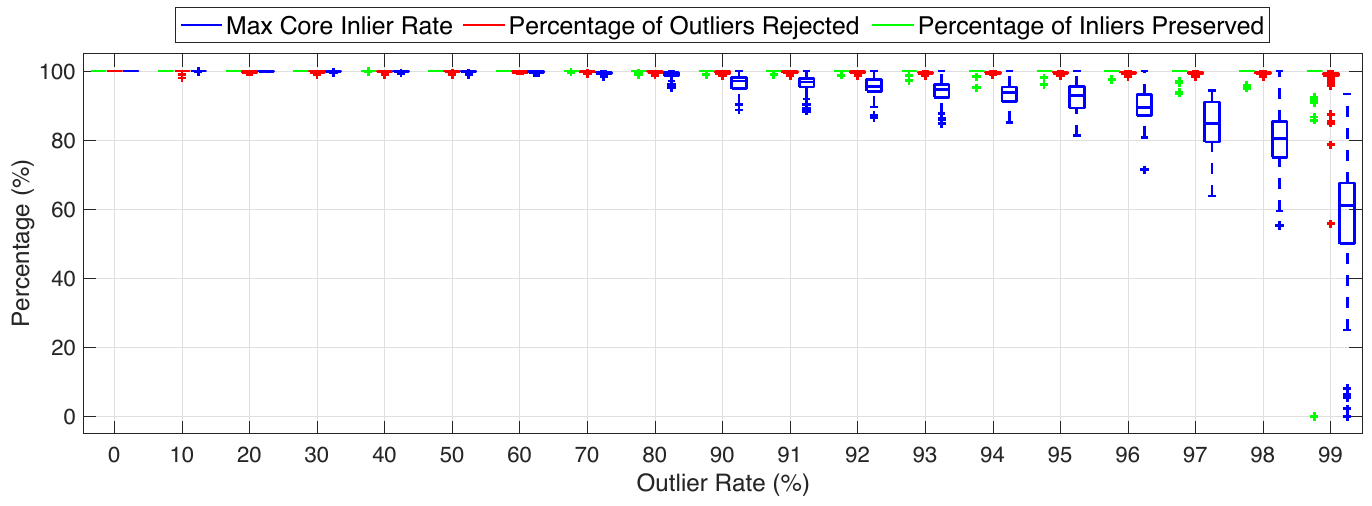} \\
 \vspace{-2mm} (c) Percentage of inliers preserved, percentage of outliers rejected, and inlier rate within the maximum $k$-core.
\end{minipage}
\vspace{1mm}
\end{tabular}
\end{minipage}
\caption{Performance of \robinstargnc and \robingnc compared to three other algorithms on single rotation averaging under increasing outlier rates.}
\label{fig:app-sra-full}
\vspace{-6mm}
\end{figure*}

%% file: sections/fig-app-PCREG-full.tex
\begin{figure*}
\centering%
\begin{minipage}{\textwidth}
\centering%
\begin{tabular}{cc}%
\multicolumn{2}{c}{
\begin{minipage}{0.3\columnwidth}%
			\centering%
			\includegraphics[width=\columnwidth]{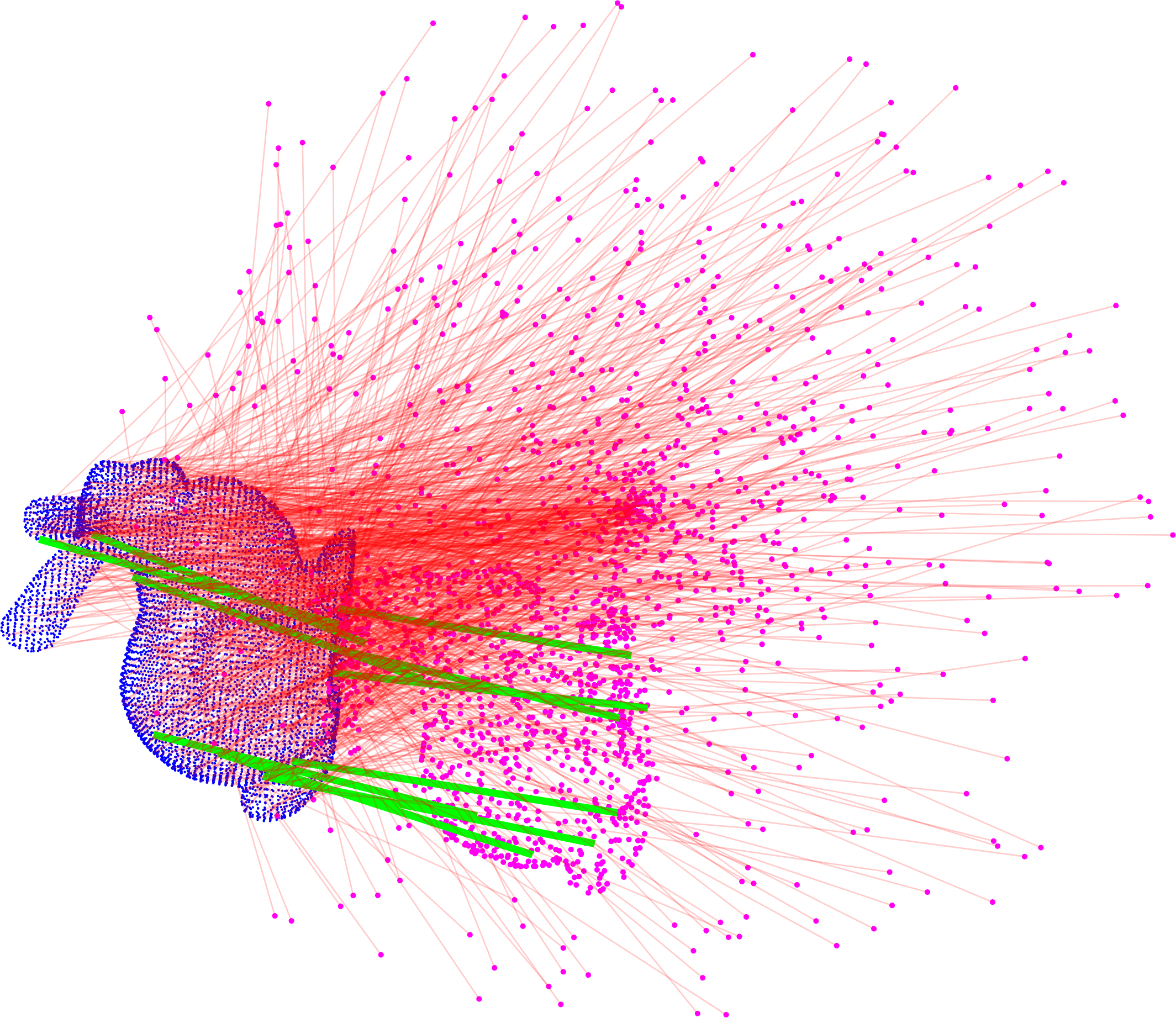} \\
(a) The \bunny model used in our tests. The green and red lines indicate inlier and outlier correspondences respectively.
\end{minipage}}
\vspace{1mm} 
\\
\begin{minipage}{0.4\columnwidth}%
			\centering%
			\includegraphics[width=\columnwidth]{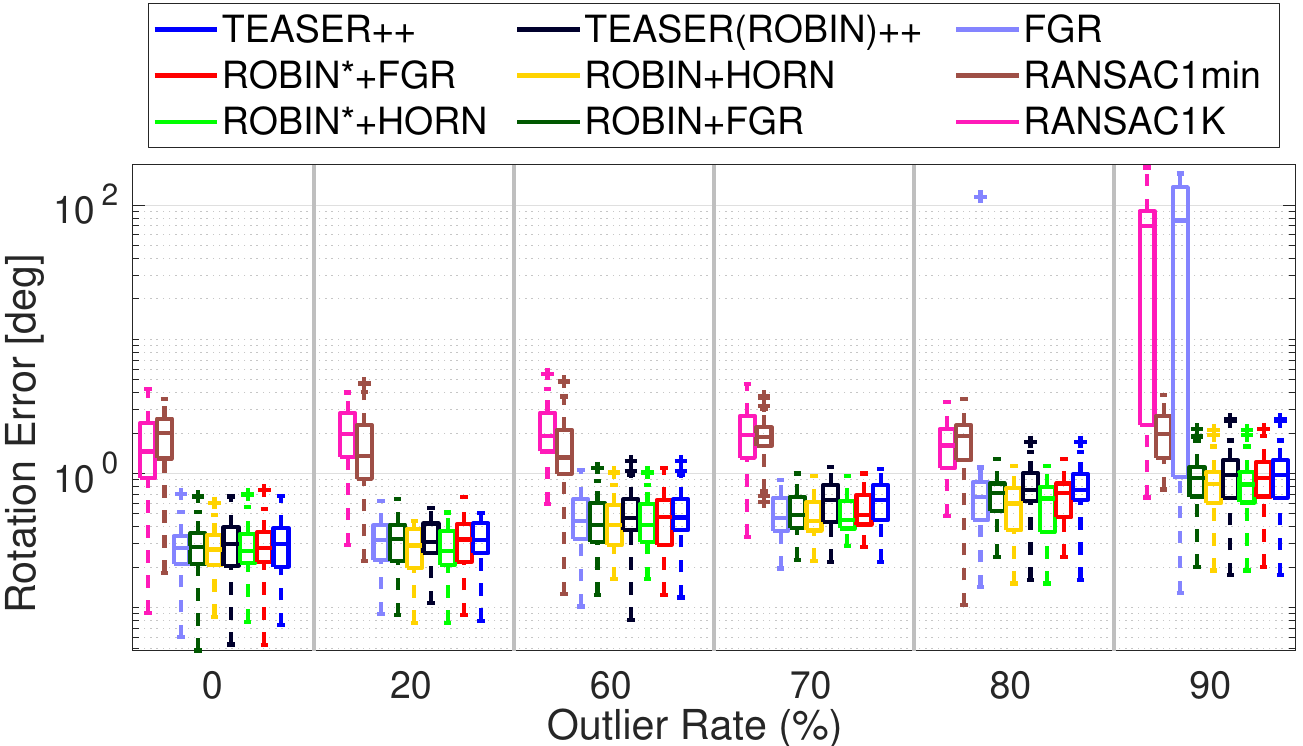} \\
 \vspace{1mm} (b) Rotation estimation error (compared to the groundtruth) \wrt~increasing outlier rates from 0\% to 90\%.
\end{minipage}
  &
\begin{minipage}{0.4\columnwidth}%
			\centering%
			\includegraphics[width=\columnwidth]{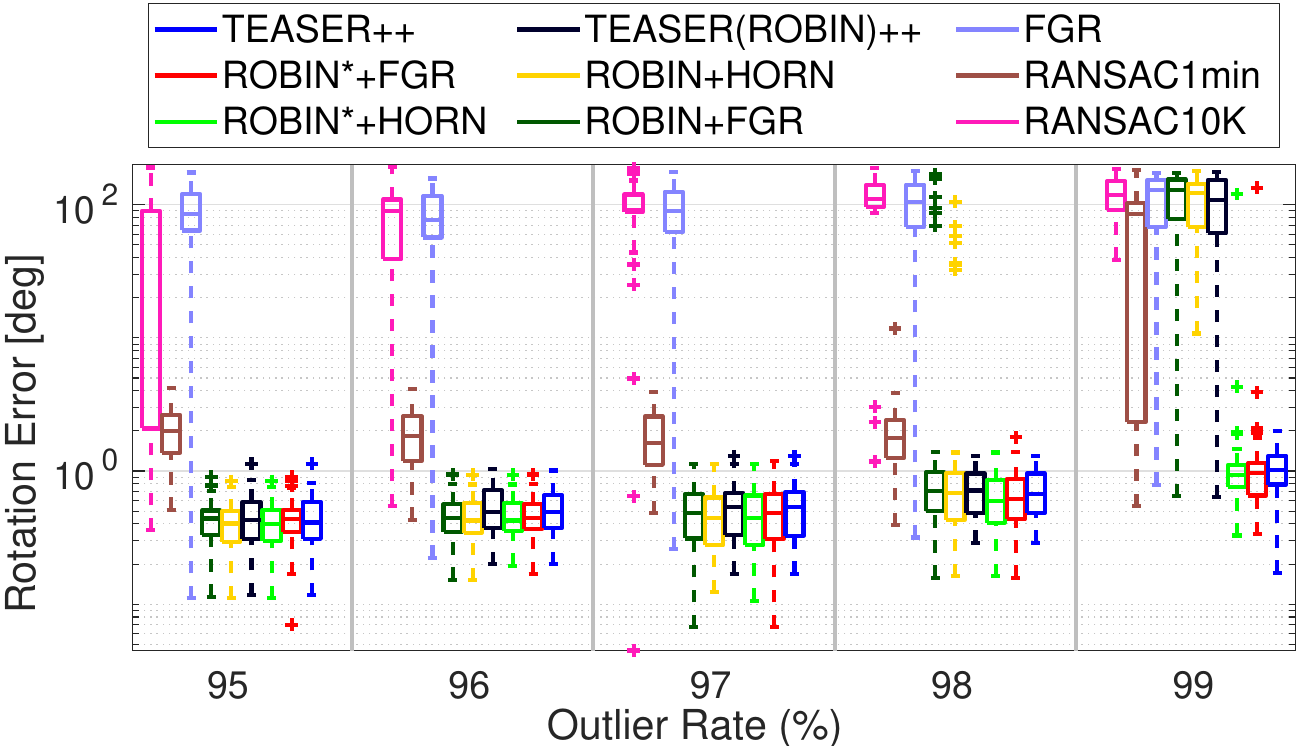} \\
 \vspace{1mm} (e) Rotation estimation error (compared to the groundtruth) \wrt~increasing outlier rates from 95\% to 99\%.
\end{minipage}
\vspace{1mm} 
\\
\begin{minipage}{0.4\columnwidth}%
			\centering%
			\includegraphics[width=\columnwidth]{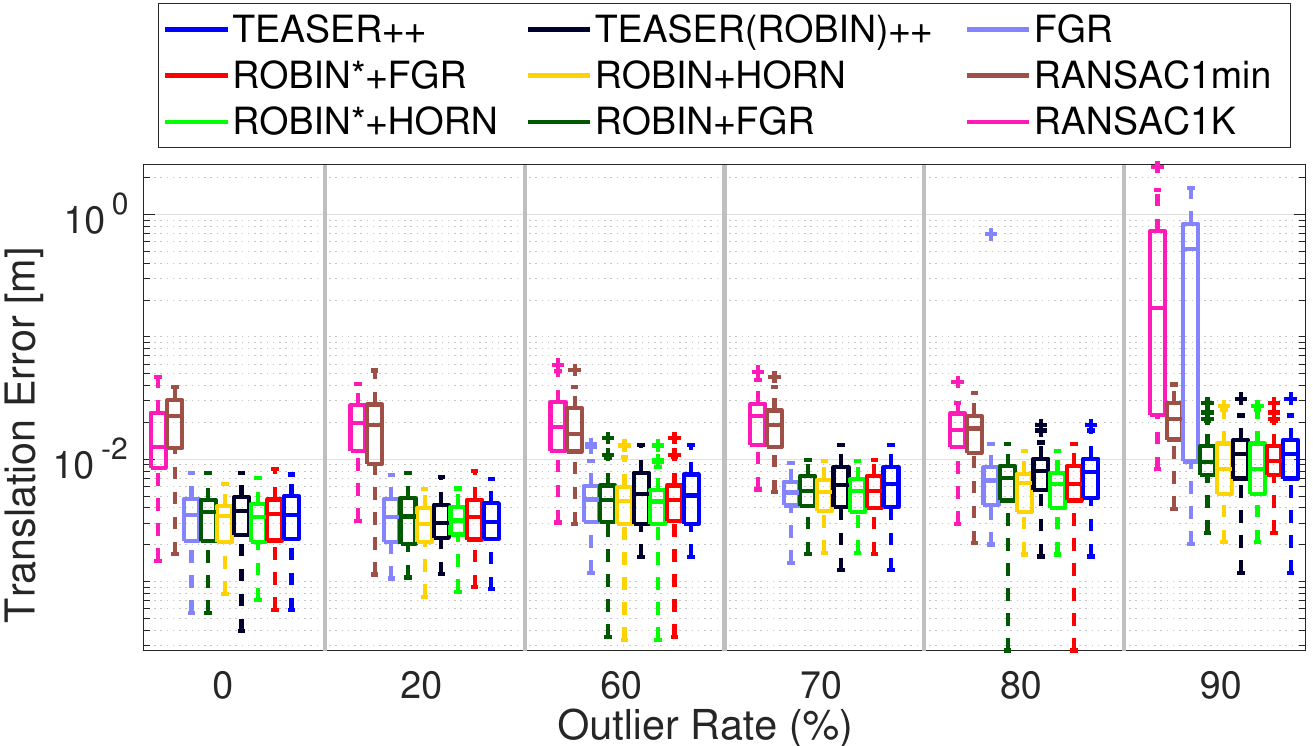} \\
 \vspace{1mm} (c) Translation estimation error (compared to the groundtruth) \wrt~increasing outlier rate from 0\% to 90\%s.
\end{minipage}
  &
\begin{minipage}{0.4\columnwidth}%
			\centering%
			\includegraphics[width=\columnwidth]{PCREG_bunny_extreme_translation_errs.pdf} \\
 \vspace{1mm} (f) Translation estimation error (compared to the groundtruth) \wrt~increasing outlier rate from 95\% to 99\%s.
\end{minipage}
\vspace{1mm} 
\\
\begin{minipage}{0.4\columnwidth}%
			\centering%
			\includegraphics[width=\columnwidth]{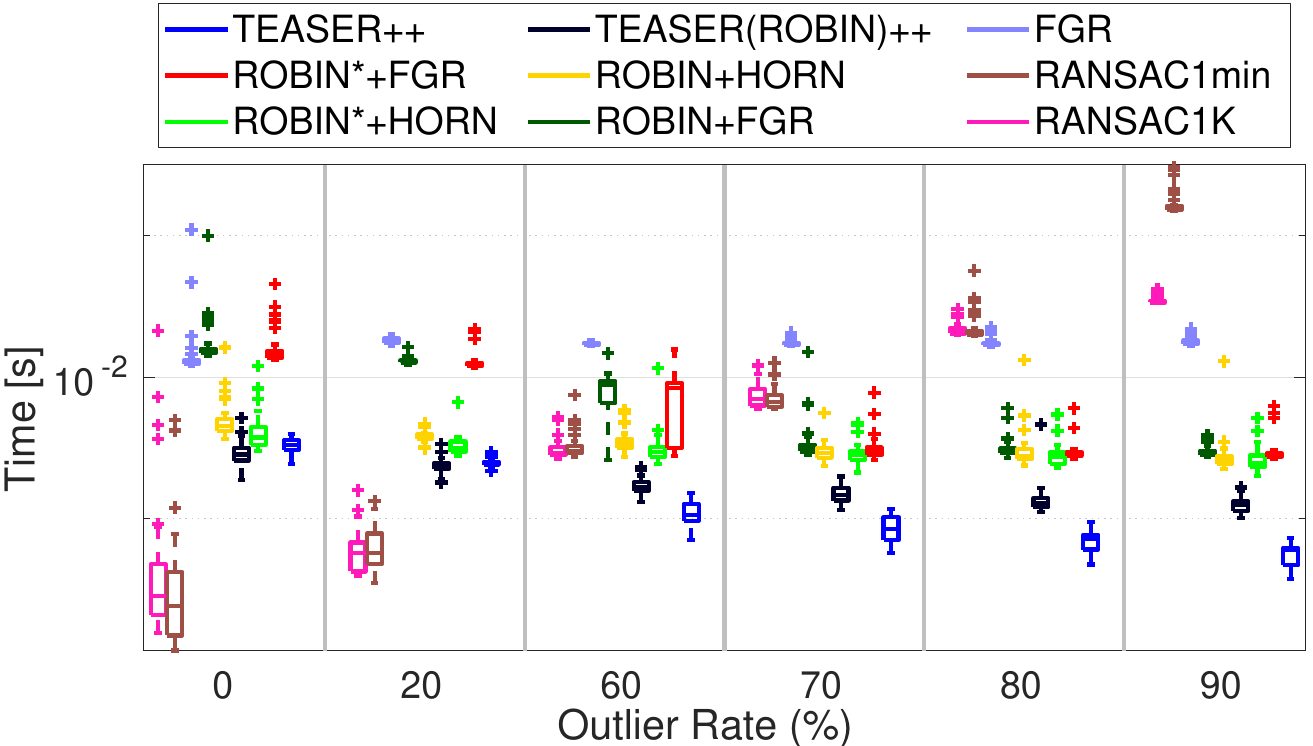} \\
 \vspace{1mm} (d) Runtime of the algorithms tested \wrt~increasing outlier rates from 0\% to 90\%.
\end{minipage}
  &
\begin{minipage}{0.4\columnwidth}%
			\centering%
			\includegraphics[width=\columnwidth]{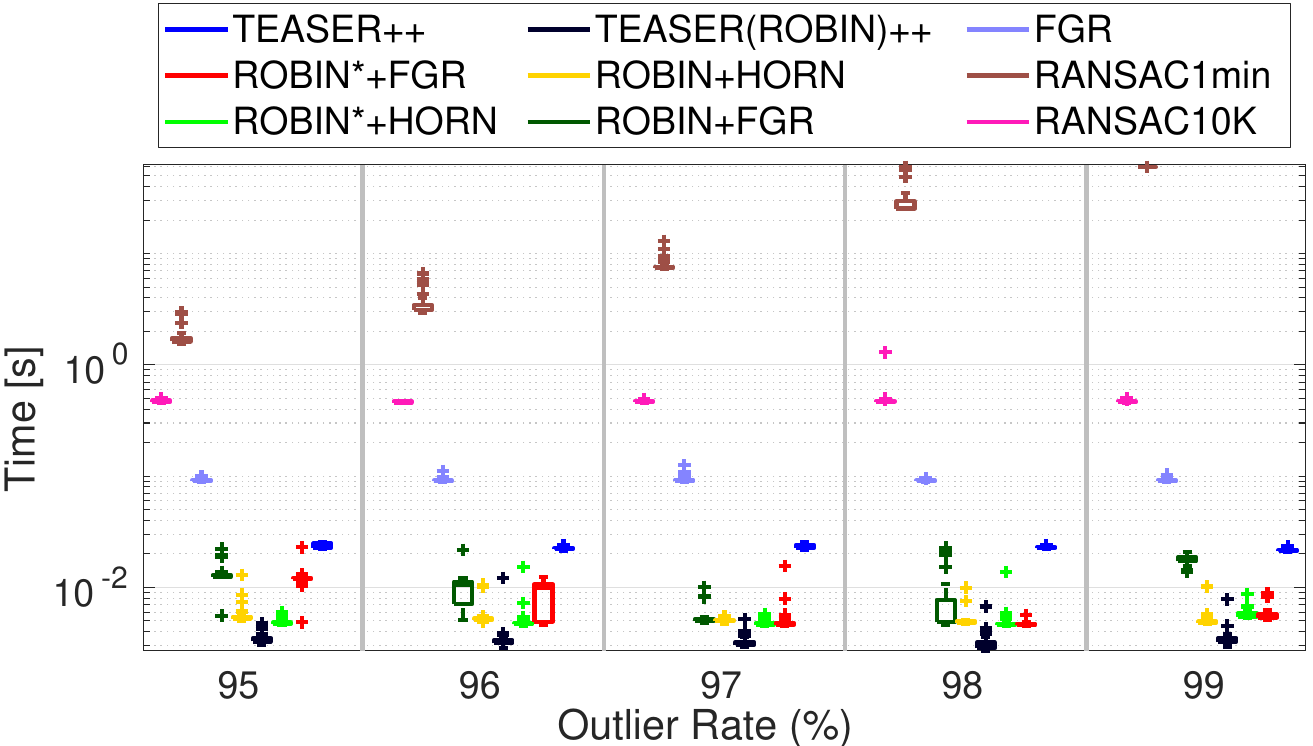} \\
 \vspace{1mm} (g) Runtime of the algorithms tested \wrt~increasing outlier rates ffrom 95\% to 99\%.
\end{minipage}
\vspace{1mm}
\end{tabular}
\end{minipage}
\caption{Performance of algorithms with \robinstar and \robin filtering compared to other algorithms on poind cloud registration with correspondences under increasing outlier rates.
  Note that \robin, \robinstar, \teaserpp and \robinteaserpp are implemented in C++. Horn's method, FGR and RANSAC are implemented in MATLAB.
}
\label{fig:app-PCREG-full}
\vspace{-6mm}
\end{figure*}

%% file: sections/fig-app-3dmatch.tex
\begin{figure*}
\centering%
\begin{minipage}{\textwidth}
\centering%
\begin{tabular}{cc}%
\begin{minipage}{0.5\columnwidth}%
			\centering%
			\includegraphics[width=\columnwidth]{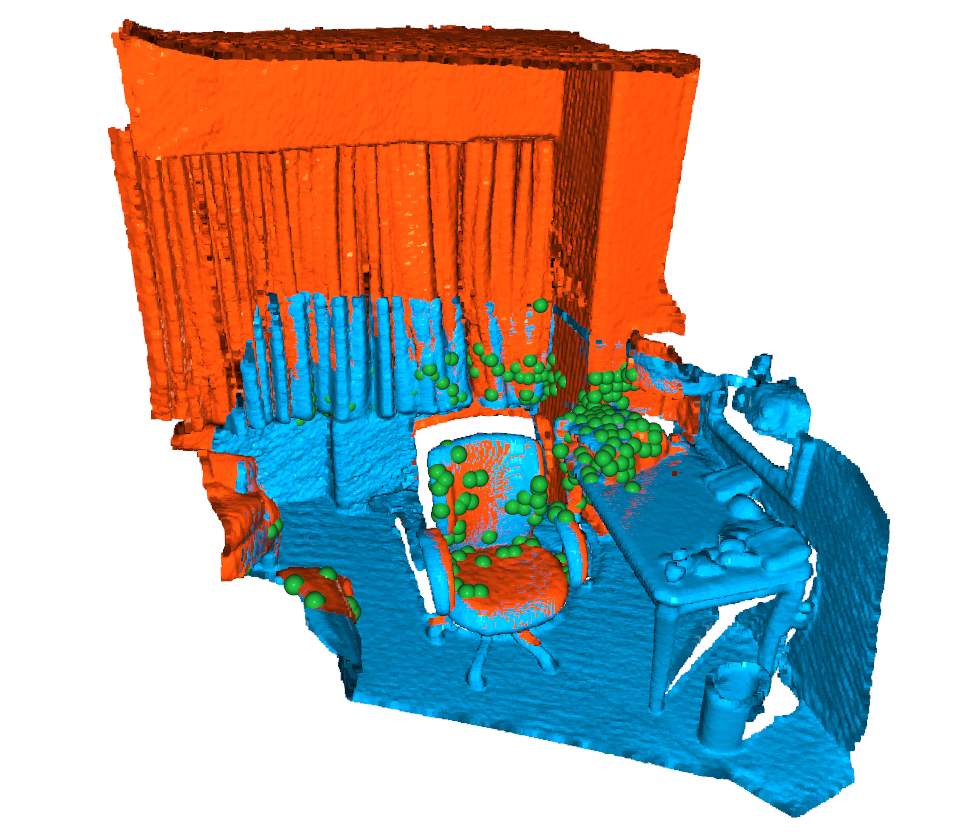} \\
 \vspace{1mm} 
\end{minipage}
  &
\begin{minipage}{0.5\columnwidth}%
			\centering%
			\includegraphics[width=\columnwidth]{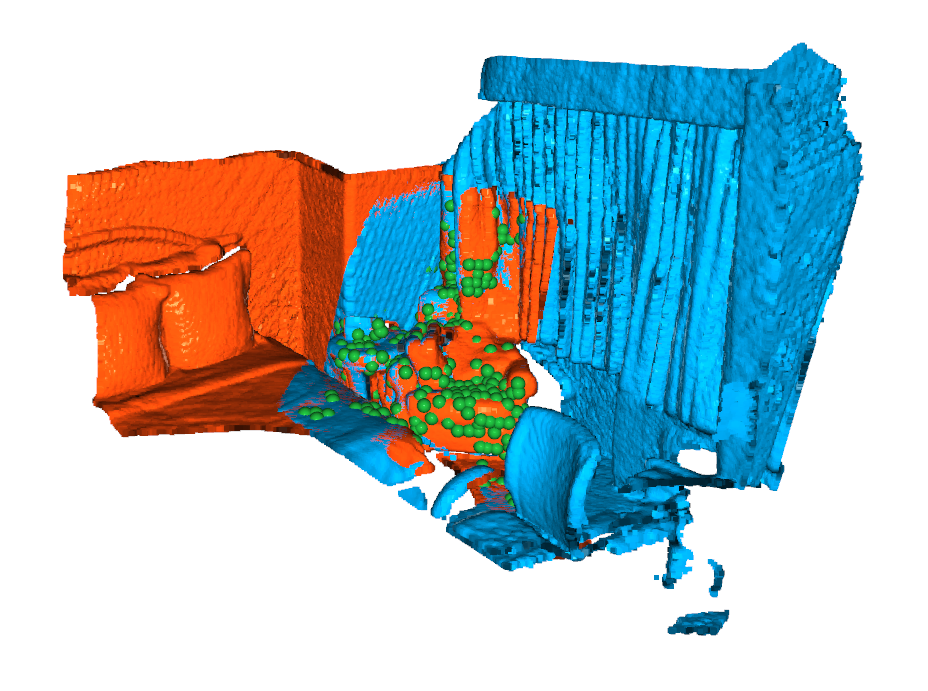} \\
 \vspace{1mm} 
\end{minipage}
\vspace{1mm} 
\\
\begin{minipage}{0.3\columnwidth}%
			\centering%
			\includegraphics[width=\columnwidth]{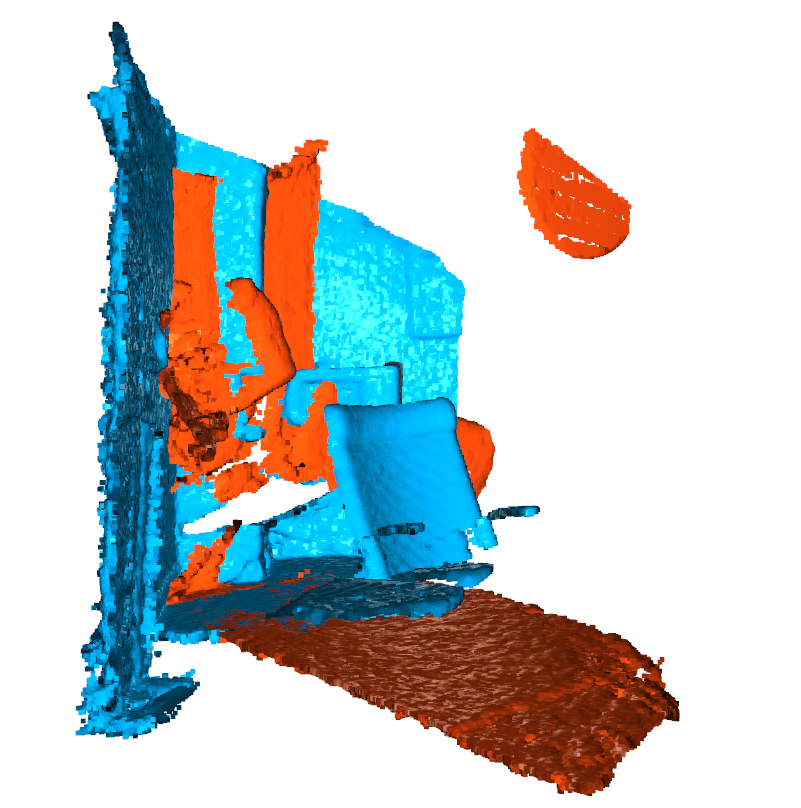} \\
 \vspace{1mm} 
\end{minipage}
  &
\begin{minipage}{0.3\columnwidth}%
			\centering%
			\includegraphics[width=\columnwidth]{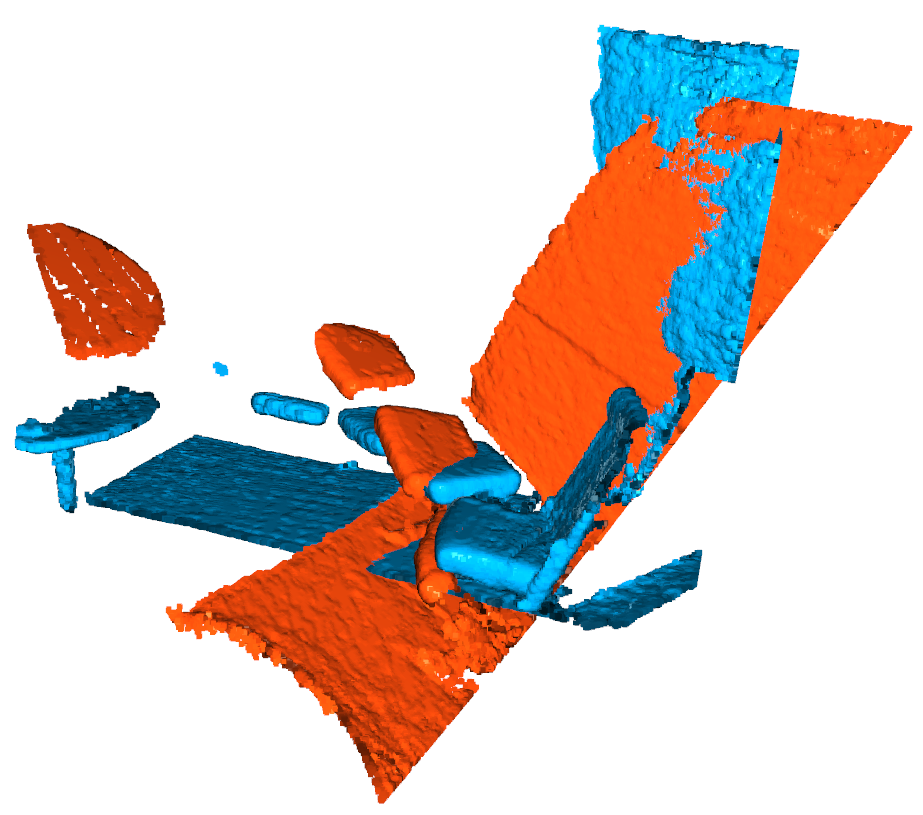} \\
 \vspace{1mm} 
\end{minipage}
\vspace{1mm} 
\vspace{1mm} 
\end{tabular}
\end{minipage}
\caption{Sample registrations from our tests on the 3DMatch dataset using \robinstargnc. The top row shows two successes, with the green spheres indicate inliers identified by \robinstar. The bottom row shows two failures.}
\label{fig:app-3dmatch-cases}
\vspace{-6mm}
\end{figure*}

%% file: sections/fig-app-SPC-full.tex
\begin{figure*}
	\centering
\begin{minipage}{\textwidth}
	\centering
\begin{tabular}{cc}%
  \multicolumn{2}{c}{
\begin{minipage}{0.2\columnwidth}%
			\centering%
			\includegraphics[width=\columnwidth]{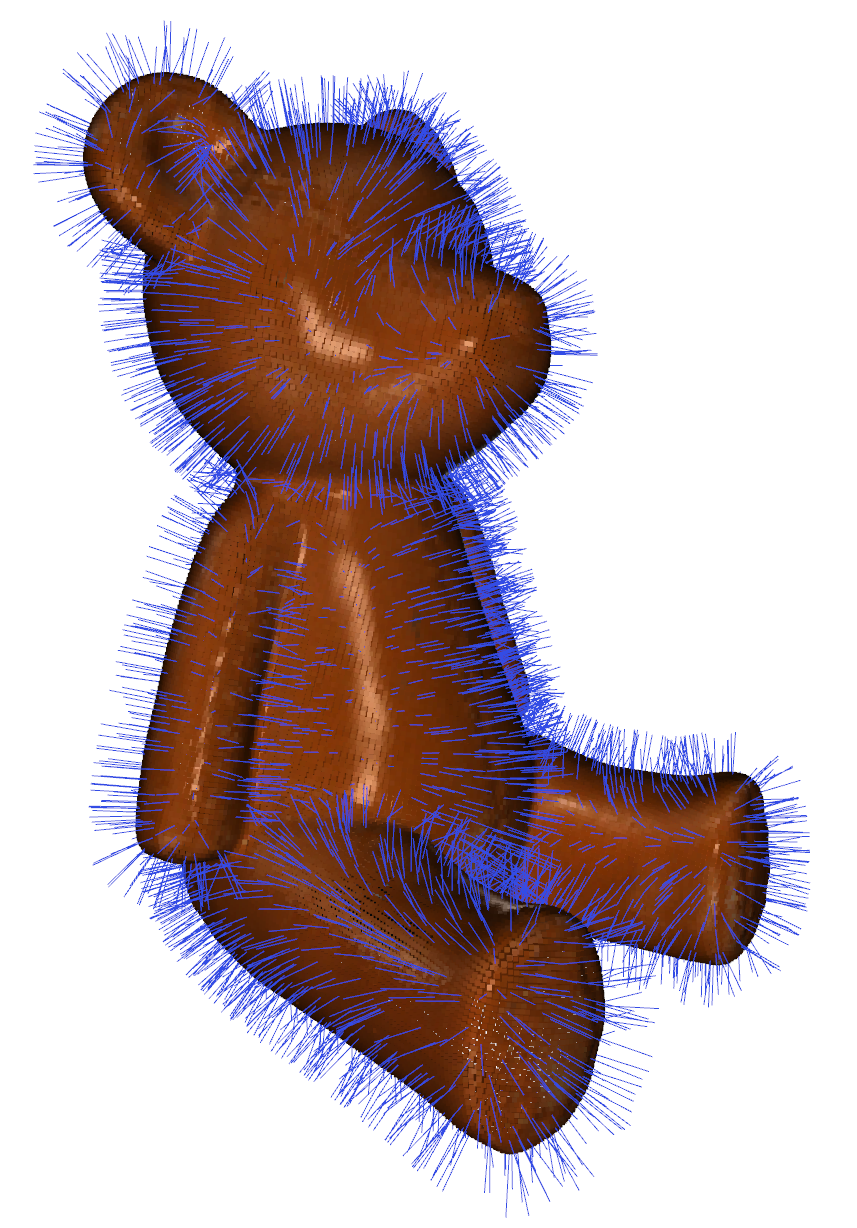} \\
 \vspace{-2mm} (a) The \teddyBear model used in our registration without correspondences tests. The blue lines indicate normals.
\end{minipage}}
\vspace{2mm}
\\
\begin{minipage}{0.5\columnwidth}%
			\centering%
			\includegraphics[width=\columnwidth]{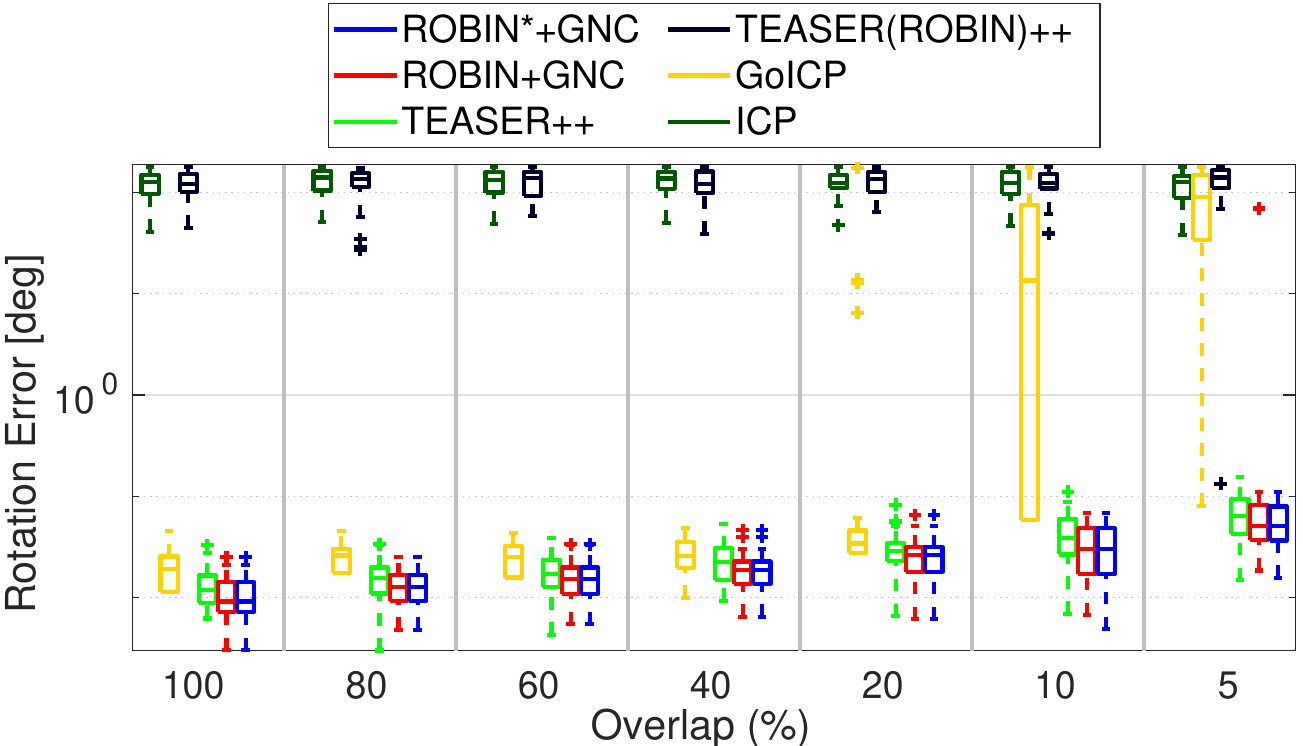} \\
 \vspace{2mm} (b) Rotation estimation error (compared to the groundtruth) \wrt~decreasing overlap ratios.
\end{minipage}
  &
\begin{minipage}{0.5\columnwidth}%
			\centering%
			\includegraphics[width=\columnwidth]{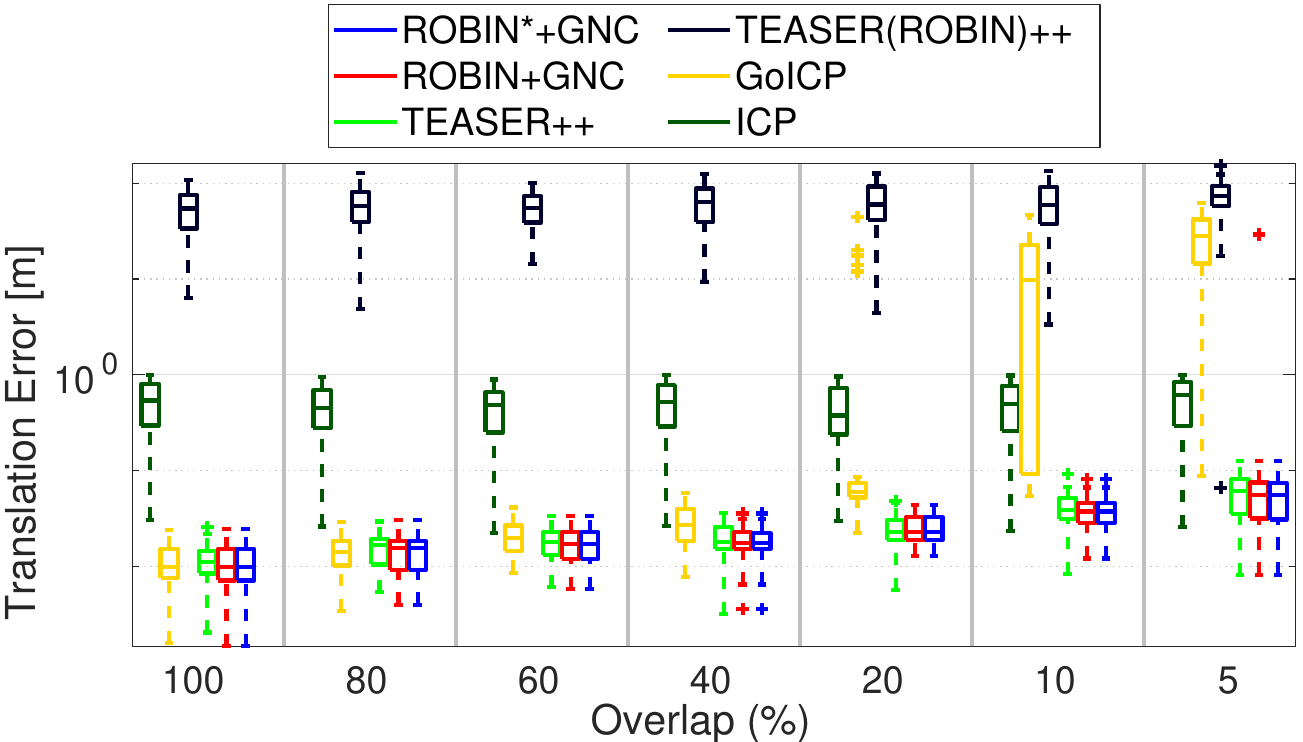} \\
 \vspace{2mm} (c) Translation estimation error (compared to the groundtruth) \wrt~decreasing overlap ratios.
\end{minipage}
\vspace{2mm}
  \\
  \multicolumn{2}{c}{
\begin{minipage}{0.5\columnwidth}%
			\centering%
			\includegraphics[width=\columnwidth]{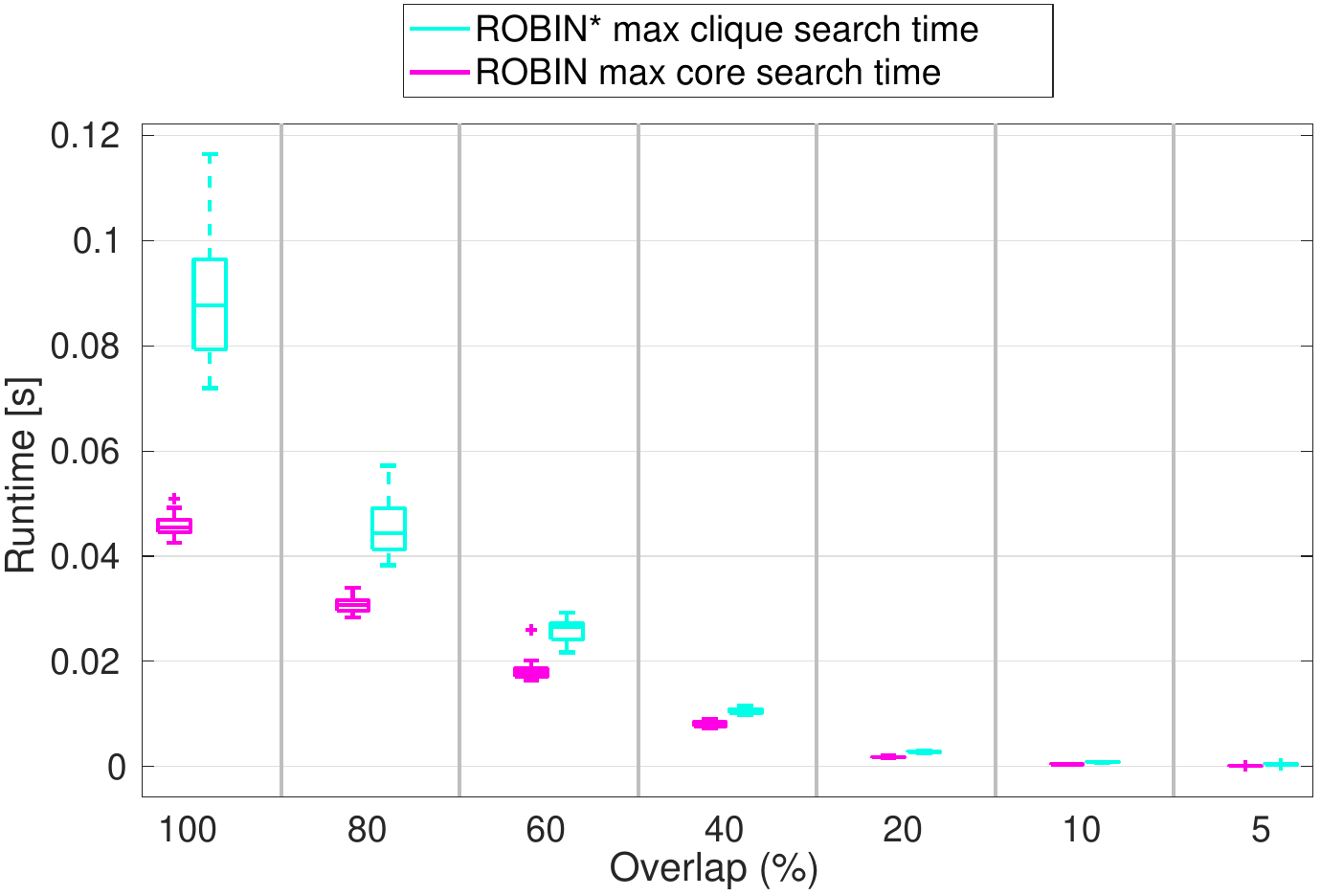} \\
 \vspace{2mm} (d) Runtime of the algorithms tested \wrt~decreasing overlap ratios.
\end{minipage}}
  \\
\end{tabular}
\end{minipage}
\caption{Performance comparison of \robinstargnc, \robingnc, \teaserpp, \robinteaserpp, ICP and \GoICP for registration without correspondences. }
\label{fig:app-SPC-full}
\vspace{-5mm}
\end{figure*}

%% file: sections/fig-app-CR-full.tex
\begin{figure*}
	\centering
\begin{minipage}{\textwidth}
	\centering
\begin{tabular}{cc}%
  \multicolumn{2}{c}{
\begin{minipage}{0.4\columnwidth}%
			\centering%
			\includegraphics[width=\columnwidth]{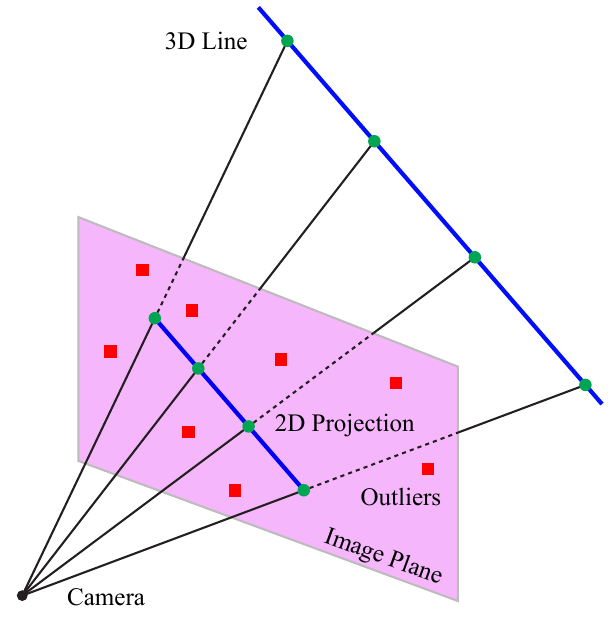} \\
 \vspace{1mm} (a) Illustration showing our testing setup for 2D-3D correspondence outlier rejection.
\end{minipage}}
\vspace{1mm}
\\
  \multicolumn{2}{c}{
\begin{minipage}{0.6\columnwidth}%
			\centering%
			\includegraphics[width=\columnwidth]{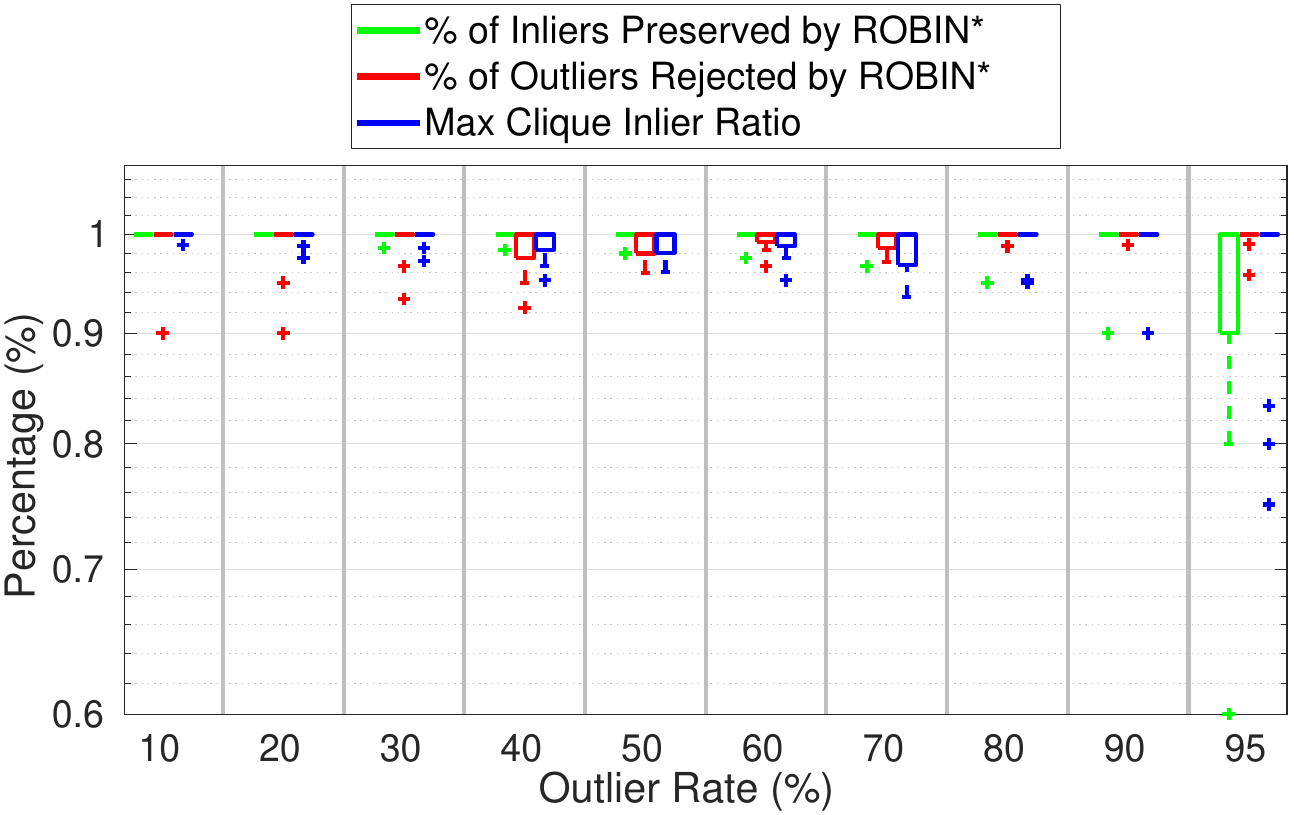} \\
 \vspace{1mm} (b) Percentage of inliers preserved, percentage of outliers rejected, and inlier rate within the maximum clique.
\end{minipage}}
\vspace{1mm}
  \\
  \multicolumn{2}{c}{
\begin{minipage}{0.6\columnwidth}%
			\centering%
			\includegraphics[width=\columnwidth]{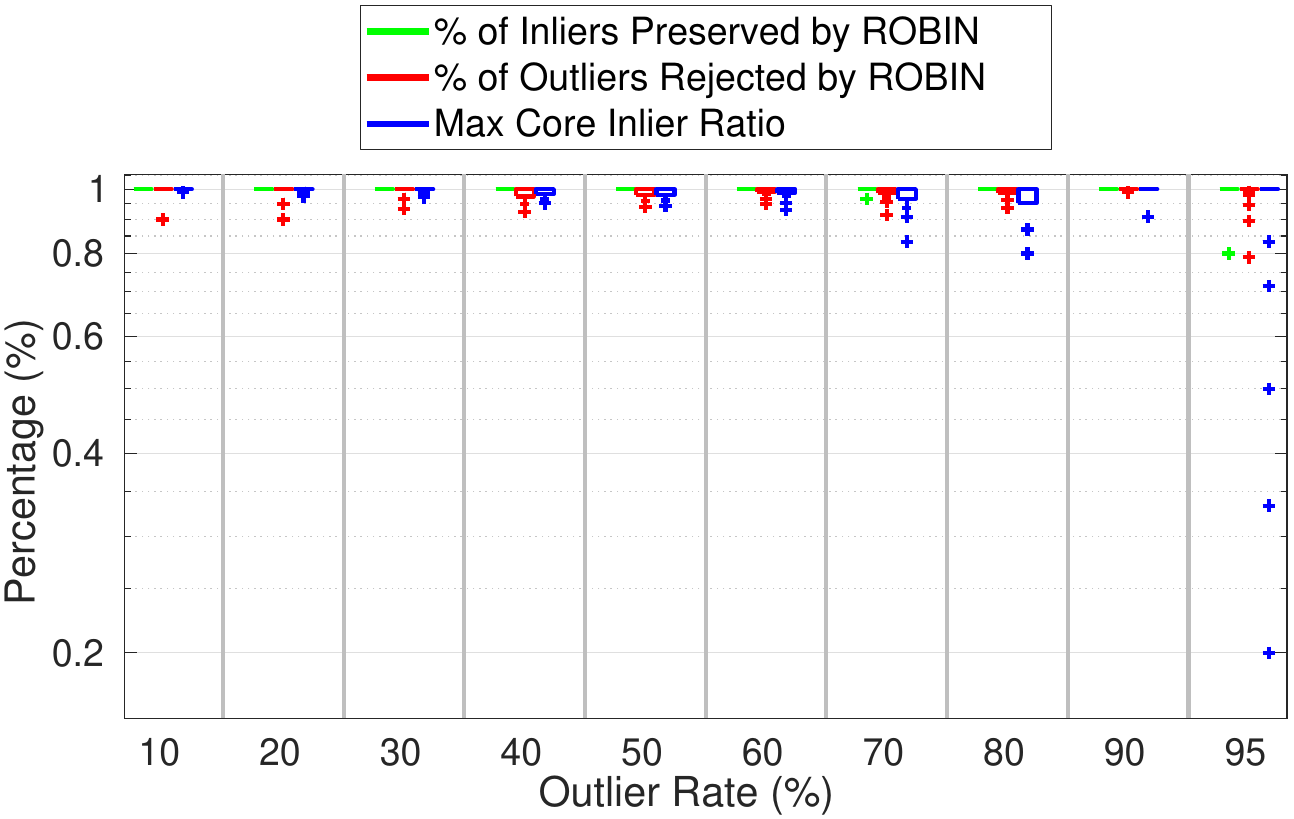} \\
 \vspace{1mm} (c) Percentage of inliers preserved, percentage of outliers rejected, and inlier rate within the maximum $k$-core.
\end{minipage}}
\vspace{1mm}
  \\
\end{tabular}
\end{minipage}
\caption{Performance of \robin and \robinstar on the 2D-3D correspondence rejection problem.}
\label{fig:app-CR-full}
\vspace{-6mm}
\end{figure*}

%% file: root.bbl
\begin{thebibliography}{10}\itemsep=-1pt

\bibitem{Abramowitz74book}
M.~Abramowitz.
\newblock {\em Handbook of Mathematical Functions, With Formulas, Graphs, and
  Mathematical Tables}.
\newblock Dover Publications, Inc., 1974.

\bibitem{Agarwal13icra}
P.~Agarwal, G.~D. Tipaldi, L.~Spinello, C.~Stachniss, and W.~Burgard.
\newblock Robust map optimization using dynamic covariance scaling.
\newblock In {\em IEEE Intl. Conf. on Robotics and Automation (ICRA)}, 2013.

\bibitem{Antonante20arxiv-outlierRobustEstimation}
P.~Antonante, V.~Tzoumas, H.~Yang, and L.~Carlone.
\newblock Outlier-robust estimation: Hardness, minimally-tuned algorithms, and
  applications.
\newblock {\em arXiv preprint arXiv: 2007.15109}, 2020.
\newblock \linkToPdf{https://arxiv.org/pdf/2007.15109.pdf}.

\bibitem{Arun87pami}
K.~Arun, T.~Huang, and S.~Blostein.
\newblock Least-squares fitting of two 3-{D} point sets.
\newblock {\em {IEEE} Trans. Pattern Anal. Machine Intell.}, 9(5):698--700,
  sept. 1987.

\bibitem{Ask13cvpr-optimalTruncatedL2}
E.~Ask, O.~Enqvist, and F.~Kahl.
\newblock Optimal geometric fitting under the truncated l2-norm.
\newblock In {\em Proceedings of the IEEE Conference on Computer Vision and
  Pattern Recognition}, pages 1722--1729, 2013.

\bibitem{Bailey00icra-dataAssociation}
T.~{Bailey}, E.~M. {Nebot}, J.~K. {Rosenblatt}, and H.~F. {Durrant-Whyte}.
\newblock Data association for mobile robot navigation: a graph theoretic
  approach.
\newblock In {\em IEEE Intl. Conf. on Robotics and Automation (ICRA)},
  volume~3, pages 2512--2517, 2000.

\bibitem{Barath18ivpr-graphCutRANSAC}
D.~Barath and J.~Matas.
\newblock Graph-cut ransac.
\newblock In {\em Proceedings of the IEEE Conference on Computer Vision and
  Pattern Recognition}, pages 6733--6741, 2018.

\bibitem{Barath18-gcRANSAC}
D.~Barath and J.~Matas.
\newblock Graph-cut ransac.
\newblock In {\em Proceedings of the IEEE Conference on Computer Vision and
  Pattern Recognition}, pages 6733--6741, 2018.

\bibitem{Barfoot17book}
T.~Barfoot.
\newblock {\em State Estimation for Robotics}.
\newblock Cambridge University Press, 2017.

\bibitem{Barron19cvpr-adaptRobustLoss}
J.~T. Barron.
\newblock A general and adaptive robust loss function.
\newblock In {\em Proceedings of the IEEE Conference on Computer Vision and
  Pattern Recognition}, pages 4331--4339, 2019.

\bibitem{Black96ijcv-unification}
M.~J. Black and A.~Rangarajan.
\newblock On the unification of line processes, outlier rejection, and robust
  statistics with applications in early vision.
\newblock {\em Intl. J. of Computer Vision}, 19(1):57--91, 1996.

\bibitem{Blake1987book-visualReconstruction}
A.~Blake and A.~Zisserman.
\newblock {\em Visual reconstruction}.
\newblock MIT Press, 1987.

\bibitem{Bosse17fnt}
M.~Bosse, G.~Agamennoni, and I.~Gilitschenski.
\newblock Robust estimation and applications in robotics.
\newblock {\em Foundations and Trends in Robotics}, 4(4):225--269, 2016.

\bibitem{Bray06eccv}
M.~Bray, P.~Kohli, and P.~Torr.
\newblock {PoseCut}: Simultaneous segmentation and 3d pose estimation of humans
  using dynamic graph-cuts.
\newblock In {\em European Conf. on Computer Vision (ECCV)}, pages 642--655,
  2006.

\bibitem{Cadena16tro-SLAMsurvey}
C.~Cadena, L.~Carlone, H.~Carrillo, Y.~Latif, D.~Scaramuzza, J.~Neira, I.~Reid,
  and J.~Leonard.
\newblock Past, present, and future of simultaneous localization and mapping:
  Toward the robust-perception age.
\newblock {\em {IEEE} Trans. Robotics}, 32(6):1309--1332, 2016.
\newblock arxiv preprint: 1606.05830,
  \linkToPdf{https://arxiv.org/abs/1606.05830}.

\bibitem{Cai19ICCV-CMtreeSearch}
Z.~Cai, T.-J. Chin, and V.~Koltun.
\newblock Consensus maximization tree search revisited.
\newblock In {\em Intl. Conf. on Computer Vision (ICCV)}, pages 1637--1645,
  2019.

\bibitem{Carlone14iros-robustPGO2D}
L.~Carlone, A.~Censi, and F.~Dellaert.
\newblock Selecting good measurements via $\ell_1$ relaxation: a convex
  approach for robust estimation over graphs.
\newblock In {\em IEEE/RSJ Intl. Conf. on Intelligent Robots and Systems
  (IROS)}, 2014.
\newblock
  \linkToPdf{https://www.dropbox.com/s/7f304d5ag245ie4/2014c-IROS-outlierRejection.pdf?dl=0}.

\bibitem{Chin18eccv-robustFitting}
T.-J. Chin, Z.~Cai, and F.~Neumann.
\newblock Robust fitting in computer vision: Easy or hard?
\newblock In {\em European Conf. on Computer Vision (ECCV)}, 2018.

\bibitem{Chin15cvpr-CMTreeAstar}
T.-J. Chin, P.~Purkait, A.~Eriksson, and D.~Suter.
\newblock Efficient globally optimal consensus maximisation with tree search.
\newblock In {\em Proceedings of the IEEE Conference on Computer Vision and
  Pattern Recognition}, pages 2413--2421, 2015.

\bibitem{Chin17slcv-maximumConsensusAdvances}
T.~J. Chin and D.~Suter.
\newblock The maximum consensus problem: recent algorithmic advances.
\newblock {\em Synthesis Lectures on Computer Vision}, 7(2):1--194, 2017.

\bibitem{Choi09cvpr-starsac}
J.~Choi and G.~Medioni.
\newblock Starsac: Stable random sample consensus for parameter estimation.
\newblock In {\em 2009 IEEE Conference on Computer Vision and Pattern
  Recognition}, pages 675--682. IEEE, 2009.

\bibitem{Chum05cvpr}
O.~Chum and J.~Matas.
\newblock Matching with {PROSAC} - progressive sample consensus.
\newblock In {\em IEEE Conf. on Computer Vision and Pattern Recognition
  (CVPR)}, 2005.

\bibitem{Dasari14-KCorePARK}
N.~S. Dasari, R.~Desh, and M.~Zubair.
\newblock Park: An efficient algorithm for k-core decomposition on multicore
  processors.
\newblock In {\em 2014 IEEE International Conference on Big Data (Big Data)},
  pages 9--16.

\bibitem{Dellaert17fnt-factorGraph}
F.~Dellaert and M.~Kaess.
\newblock Factor graphs for robot perception.
\newblock {\em Foundations and Trends in Robotics}, 6(1-2):1--139, 2017.

\bibitem{Enqvist12eccv-robustFitting}
O.~Enqvist, E.~Ask, F.~Kahl, and K.~{\AA}str{\"o}m.
\newblock Robust fitting for multiple view geometry.
\newblock In {\em European Conf. on Computer Vision (ECCV)}, pages 738--751.
  Springer, 2012.

\bibitem{Enqvist09iccv}
O.~{Enqvist}, K.~{Josephson}, and F.~{Kahl}.
\newblock Optimal correspondences from pairwise constraints.
\newblock In {\em Intl. Conf. on Computer Vision (ICCV)}, pages 1295--1302,
  2009.

\bibitem{Fischler81}
M.~Fischler and R.~Bolles.
\newblock Random sample consensus: a paradigm for model fitting with
  application to image analysis and automated cartography.
\newblock {\em Commun. ACM}, 24:381--395, 1981.

\bibitem{Gros92-projectiveInvariantsTheory}
P.~Gros and L.~Quan.
\newblock Projective invariants for vision.
\newblock 1992.

\bibitem{Hartley13ijcv}
R.~Hartley, J.~Trumpf, Y.~Dai, and H.~Li.
\newblock Rotation averaging.
\newblock {\em IJCV}, 103(3):267--305, 2013.

\bibitem{Hartley04}
R.~I. Hartley and A.~Zisserman.
\newblock {\em Multiple View Geometry in Computer Vision}.
\newblock Cambridge University Press, second edition, 2004.

\bibitem{Horn87josa}
B.~K.~P. Horn.
\newblock Closed-form solution of absolute orientation using unit quaternions.
\newblock {\em J. Opt. Soc. Amer.}, 4(4):629--642, Apr 1987.

\bibitem{Huber81}
P.~Huber.
\newblock {\em Robust Statistics}.
\newblock John Wiley \& Sons, New York, NY, 1981.

\bibitem{Izatt17isrr-MIPregistration}
G.~Izatt, H.~Dai, and R.~Tedrake.
\newblock Globally optimal object pose estimation in point clouds with
  mixed-integer programming.
\newblock In {\em Proc. of the Intl. Symp. of Robotics Research (ISRR)}, 2017.

\bibitem{Kabir17IPDPSW-KCorePKC}
H.~Kabir and K.~Madduri.
\newblock Parallel k-core decomposition on multicore platforms.
\newblock In {\em 2017 IEEE International Parallel and Distributed Processing
  Symposium Workshops (IPDPSW)}, pages 1482--1491. IEEE, 2017.

\bibitem{Kaskman19-homebrewedDB}
R.~Kaskman, S.~Zakharov, I.~Shugurov, and S.~Ilic.
\newblock Homebreweddb: Rgb-d dataset for 6d pose estimation of 3d objects.
\newblock In {\em Proceedings of the IEEE International Conference on Computer
  Vision Workshops}, pages 0--0, 2019.

\bibitem{Lajoie19ral-DCGM}
P.~Lajoie, S.~Hu, G.~Beltrame, and L.~Carlone.
\newblock Modeling perceptual aliasing in {SLAM} via discrete-continuous
  graphical models.
\newblock {\em {IEEE} Robotics and Automation Letters ({RA-L})}, 2019.
\newblock extended ArXiv version:
  \linkToPdf{https://arxiv.org/pdf/1810.11692.pdf}, Supplemental Material:
  \linkToPdf{https://www.dropbox.com/s/vupak65wi75yzbl/2018j-RAL-DCGM-supplemental.pdf?dl=0}.

\bibitem{Lee20arXiv-robustSRA}
S.~H. Lee and J.~Civera.
\newblock Robust single rotation averaging.
\newblock {\em arXiv preprint arXiv:2004.00732}, 2020.

\bibitem{Leordeanu05-spectral}
M.~Leordeanu and M.~Hebert.
\newblock A spectral technique for correspondence problems using pairwise
  constraints.
\newblock In {\em Tenth IEEE International Conference on Computer Vision
  (ICCV'05) Volume 1}, volume~2, pages 1482--1489. IEEE, 2005.

\bibitem{Li09cvpr-robustFitting}
H.~Li.
\newblock Consensus set maximization with guaranteed global optimality for
  robust geometry estimation.
\newblock In {\em Intl. Conf. on Computer Vision (ICCV)}, pages 1074--1080,
  2009.

\bibitem{Mangelson18icra}
J.~G. Mangelson, D.~Dominic, R.~M. Eustice, and R.~Vasudevan.
\newblock Pairwise consistent measurement set maximization for robust
  multi-robot map merging.
\newblock In {\em IEEE Intl. Conf. on Robotics and Automation (ICRA)}, pages
  2916--2923, 2018.

\bibitem{markley1988jas-svdAttitudeDeter}
F.~L. Markley.
\newblock Attitude determination using vector observations and the singular
  value decomposition.
\newblock {\em The Journal of the Astronautical Sciences}, 36(3):245--258,
  1988.

\bibitem{Mundy92book}
J.~Mundy and A.~Zisserman.
\newblock {\em Geometric invariance in computer vision}.
\newblock MIT Press, Cambridge, MA, USA, 1992.

\bibitem{Neira01tra}
J.~Neira and J.~Tard{\'o}s.
\newblock Data association in stochastic mapping using the joint compatibility
  test.
\newblock {\em {IEEE} Trans. Robot. Automat.}, 17(6):890--897, December 2001.

\bibitem{Bustos2015iccv-gore3D}
{\'A}.~{Parra Bustos} and T.~J. Chin.
\newblock Guaranteed outlier removal for rotation search.
\newblock In {\em Proceedings of the IEEE International Conference on Computer
  Vision}, pages 2165--2173, 2015.

\bibitem{Bustos18pami-GORE}
{\'A}.~{Parra Bustos} and T.~J. Chin.
\newblock Guaranteed outlier removal for point cloud registration with
  correspondences.
\newblock {\em {IEEE} Trans. Pattern Anal. Machine Intell.}, 40(12):2868--2882,
  2018.

\bibitem{Parra19arXiv-practicalMaxClique}
A.~Parra~Bustos, T.-J. Chin, F.~Neumann, T.~Friedrich, and M.~Katzmann.
\newblock A practical maximum clique algorithm for matching with pairwise
  constraints.
\newblock {\em arXiv preprint arXiv:1902.01534}, 2019.

\bibitem{Pavlakos17icra-semanticKeypoints}
G.~Pavlakos, X.~Zhou, A.~Chan, K.~G. Derpanis, and K.~Daniilidis.
\newblock 6-dof object pose from semantic keypoints.
\newblock In {\em IEEE Intl. Conf. on Robotics and Automation (ICRA)}, 2017.

\bibitem{Perera12-maxCliqueSegmentation}
S.~Perera and N.~Barnes.
\newblock Maximal cliques based rigid body motion segmentation with a rgb-d
  camera.
\newblock In {\em Asian Conference on Computer Vision}, pages 120--133.
  Springer, 2012.

\bibitem{Raguram12pami-usac}
R.~Raguram, O.~Chum, M.~Pollefeys, J.~Matas, and J.-M. Frahm.
\newblock Usac: a universal framework for random sample consensus.
\newblock {\em IEEE transactions on pattern analysis and machine intelligence},
  35(8):2022--2038, 2012.

\bibitem{Raguram08-RANSACcomparative}
R.~Raguram, J.-M. Frahm, and M.~Pollefeys.
\newblock A comparative analysis of ransac techniques leading to adaptive
  real-time random sample consensus.
\newblock In {\em European Conference on Computer Vision}, pages 500--513.
  Springer, 2008.

\bibitem{Rosinol20rss-dynamicSceneGraphs}
A.~Rosinol, A.~Gupta, M.~Abate, J.~Shi, and L.~Carlone.
\newblock {3D} dynamic scene graphs: Actionable spatial perception with places,
  objects, and humans.
\newblock In {\em Robotics: Science and Systems (RSS)}, 2020.
\newblock \linkToPdf{https://arxiv.org/pdf/2002.06289.pdf},
  \linkToVideo{https://www.youtube.com/watch?v=SWbofjhyPzI&feature=youtu.be}.

\bibitem{Rossi15parallel}
R.~A. Rossi, D.~F. Gleich, and A.~H. Gebremedhin.
\newblock Parallel maximum clique algorithms with applications to network
  analysis.
\newblock {\em SIAM Journal on Scientific Computing}, 37(5):C589--C616, 2015.

\bibitem{Schonberger16cvpr-SfMRevisited}
J.~L. Schonberger and J.-M. Frahm.
\newblock Structure-from-motion revisited.
\newblock In {\em IEEE Conf. on Computer Vision and Pattern Recognition
  (CVPR)}, pages 4104--4113, 2016.

\bibitem{Tanaka06icra-incrementalRANSAC}
K.~Tanaka and E.~Kondo.
\newblock Incremental ransac for online relocation in large dynamic
  environments.
\newblock In {\em Proceedings 2006 IEEE International Conference on Robotics
  and Automation, 2006. ICRA 2006.}, pages 68--75. IEEE, 2006.

\bibitem{MacTavish15crv-robustEstimation}
K.~M. Tavish and T.~D. Barfoot.
\newblock At all costs: A comparison of robust cost functions for camera
  correspondence outliers.
\newblock In {\em Computer and Robot Vision (CRV), 2015 12th Conference on},
  pages 62--69. IEEE, 2015.

\bibitem{Torr00cviu}
P.~Torr and A.~Zisserman.
\newblock {MLESAC}: A new robust estimator with application to estimating image
  geometry.
\newblock {\em Comput. Vis. Image Underst.}, 78(1):138--156, 2000.

\bibitem{Walteros20OR-mcandmk}
J.~L. Walteros and A.~Buchanan.
\newblock Why is maximum clique often easy in practice?
\newblock {\em Operations Research}, 2020.

\bibitem{Yang20ral-GNC}
H.~Yang, P.~Antonante, V.~Tzoumas, and L.~Carlone.
\newblock Graduated non-convexity for robust spatial perception: From
  non-minimal solvers to global outlier rejection.
\newblock {\em {IEEE} Robotics and Automation Letters ({RA-L})},
  5(2):1127--1134, 2020.
\newblock arXiv preprint arXiv:1909.08605 (with supplemental material),
  \linkToPdf{https://arxiv.org/pdf/1909.08605.pdf} \award{,ICRA Best paper
  award in Robot Vision}.

\bibitem{Yang19rss-teaser}
H.~Yang and L.~Carlone.
\newblock A polynomial-time solution for robust registration with extreme
  outlier rates.
\newblock In {\em Robotics: Science and Systems (RSS)}, 2019.
\newblock
  \linkToPdf{http://rss2019.informatik.uni-freiburg.de/papers/0013_FI.pdf},
  \linkToVideo{http://rss2019.informatik.uni-freiburg.de/videos/0013_VI_fi.mp4},
  \linkToMedia{http://news.mit.edu/2019/spotting-objects-cars-robots-0620},
  \linkToMedia{https://www.sciencedaily.com/releases/2019/06/190620121444.htm},
  \linkToMedia{http://www.ansa.it/canale_scienza_tecnica/notizie/tecnologie/2019/06/21/i-robot-imparano-a-vedere-nella-nebbia-_9e59485c-ff17-4d62-8224-1d42f44111b9.html}.

\bibitem{Yang19iccv-QUASAR}
H.~Yang and L.~Carlone.
\newblock A quaternion-based certifiably optimal solution to the {Wahba}
  problem with outliers.
\newblock In {\em Intl. Conf. on Computer Vision (ICCV)}, 2019.
\newblock (Oral Presentation, accept rate: 4\%), Arxiv version: 1905.12536,
  \linkToPdf{https://arxiv.org/pdf/1905.12536.pdf}.

\bibitem{Yang20cvpr-shapeStar}
H.~Yang and L.~Carlone.
\newblock In perfect shape: Certifiably optimal {3D} shape reconstruction from
  {2D} landmarks.
\newblock In {\em IEEE Conf. on Computer Vision and Pattern Recognition
  (CVPR)}, 2020.
\newblock Arxiv version: 1911.11924,
  \linkToPdf{https://arxiv.org/pdf/1911.11924.pdf}.

\bibitem{Yang20neurips-certifiablePerception}
H.~Yang and L.~Carlone.
\newblock One ring to rule them all: Certifiably robust geometric perception
  with outliers.
\newblock In {\em Conference on Neural Information Processing Systems
  (NeurIPS)}, 2020.
\newblock \linkToPdf{https://arxiv.org/pdf/2006.06769.pdf}.

\bibitem{Yang20tro-teaser}
H.~Yang, J.~Shi, and L.~Carlone.
\newblock {TEASER: Fast and Certifiable Point Cloud Registration}.
\newblock {\em {IEEE} Trans. Robotics}, 2020.
\newblock extended arXiv version 2001.07715
  \linkToPdf{https://arxiv.org/pdf/2001.07715.pdf}.

\bibitem{Yang16pami-goicp}
J.~Yang, H.~Li, D.~Campbell, and Y.~Jia.
\newblock {Go-ICP}: A globally optimal solution to {3D ICP} point-set
  registration.
\newblock {\em {IEEE} Trans. Pattern Anal. Machine Intell.}, 38(11):2241--2254,
  Nov. 2016.

\bibitem{Zeng17cvpr-3dmatch}
A.~Zeng, S.~Song, M.~Nie{\ss}ner, M.~Fisher, J.~Xiao, and T.~Funkhouser.
\newblock 3dmatch: Learning the matching of local 3d geometry in range scans.
\newblock In {\em Proceedings of the IEEE Conference on Computer Vision and
  Pattern Recognition}, volume~1, page~4, 2017.

\bibitem{Zhou16eccv-fastGlobalRegistration}
Q.~Zhou, J.~Park, and V.~Koltun.
\newblock Fast global registration.
\newblock In {\em European Conf. on Computer Vision (ECCV)}, pages 766--782.
  Springer, 2016.

\bibitem{Zhou18arxiv-open3D}
Q.-Y. Zhou, J.~Park, and V.~Koltun.
\newblock {Open3D}: {A} modern library for {3D} data processing.
\newblock {\em arXiv:1801.09847}, 2018.

\end{thebibliography}
